\documentclass[12pt,onecolumn]{IEEEtran} %
\usepackage{cite}
\usepackage{amsmath,amssymb,amsfonts}
\usepackage{algorithmic}
\usepackage{graphicx}
\usepackage{textcomp}
\usepackage{xcolor}

\usepackage{bm}
\usepackage{bbm}
\usepackage{subfig}
\usepackage{color, colortbl}
\usepackage{soul}
\usepackage{url}
\usepackage{enumitem} % more options for enumeration 
\usepackage[ruled,vlined]{algorithm2e} % Algorithm package 
\usepackage{amsmath}
\usepackage{amsthm}
\usepackage{tikz} % <--- TIKZ
\usepackage{pgfplots}
\usepackage[normalem]{ulem}
\pgfplotsset{compat=1.5}

\usepackage{filecontents}

\usepackage{floatflt}
\usepackage{graphics}

\usetikzlibrary{arrows}

\newtheorem{lemma}{Lemma}
\newtheorem{corollary}{Corollary}

\def\conv{\otimes}
\def\deconv{\oslash}
\def\E{{\mathcal E}}  %%%   for "expectation value" (in math mode)

\begin{document}
%%
%% The "title" command has an optional parameter,
%% allowing the author to define a "short title" to be used in page headers.
\title{On the Burstiness of Distributed Machine Learning Traffic}
%%
%% The "author" command and its associated commands are used to define
%% the authors and their affiliations.
%% Of note is the shared affiliation of the first two authors, and the
%% "authornote" and "authornotemark" commands
%% used to denote shared contribution to the research.
\author{Natchanon Luangsomboon,~Fahimeh Fazel,~J\"{o}rg Liebeherr\\Ashkan Sobhani,~Shichao Guan,~Xingjun Chu
\thanks{
N. Luangsomboon, F. Fazel, and J. Liebeherr  are with the Department of Electrical and Computer Engineering, University of Toronto. A. Sobhani, S. Guan, and Xingjun Chu are with Huawei Canada, Ottawa. 
}
}

\maketitle
%%
%% The abstract is a short summary of the work to be presented in the
%% article.
\begin{abstract}
Traffic from distributed training of machine learning (ML) models makes up a large and growing fraction of the traffic mix in enterprise data centers. While work on distributed ML abounds,  the network traffic generated by distributed ML has received little attention. 
Using measurements on a testbed network, we investigate the traffic characteristics 
generated by  the training of the ResNet-50 neural network 
with an emphasis on  studying its short-term burstiness. 
For the latter we propose metrics that 
quantify traffic burstiness at different time scales.
Our analysis reveals that distributed ML traffic exhibits a very high degree of burstiness on 
short time scales, exceeding a 60:1 peak-to-mean ratio on time intervals as long as  5~ms. We observe that training software orchestrates 
transmissions in such a way that burst transmissions from different sources 
within the same application do not result 
in congestion and packet losses. 
An extrapolation of the measurement data to multiple applications 
underscores the challenges of distributed ML traffic 
for congestion and flow  control algorithms. 
\end{abstract}

%------------------------------------------------------------------
\section{Introduction}
\label{sec:intro}
This paper studies and analyzes the burstiness  of traffic from training deep neural network (DNN) models as a root cause for short-lived surges of traffic, known as microbursts, 
that cause periods of high packet delay and loss in a data center network (DCN) even at a low utilization. 
Since microbursts occur at a time scale of less than 
a millisecond~\cite{microburst1}, traditional traffic control methods 
are not effective with avoiding packet losses in such scenarios. 
Research on microbursts in DCNs has suggested a range of potential 
root causes,  including the inherent burstiness of application traffic, 
confluence of traffic flows to a common destination (fan-in, incast), 
offloading of protocol processing at hosts, and traffic control algorithms, such as packet scheduling and  flow control \cite{microburst1,microburst2,microburst5,microburst6,microburst2-1, microburst3,microburst4,microburst7,microburst8,microburst9}. 
While training of neural networks makes up a large fraction of the workload 
in data centers \cite{facebook-workload-20}, %\remark{Need more citations to support this.} 
to the best of our knowledge, there does not exist a detailed analysis of distributed ML traffic and 
its potential impact on the creation of microbursts.   

The vast majority of network traffic from training DNN models is due to  
the exchange of gradients of model parameters. As modern DNN models involve millions, 
and, in the case of large language models such as GPT, billions  of parameters~\cite{DNNsize}, 
the transmission of gradients  creates  huge data bursts. 
In this paper, we present measurement experiments of the training of ResNet-50, a convolutional neural network for image classification with a total of 25 million parameters~\cite{resnet50}. The measurement experiments are performed in a testbed network with a single switch with 100~Gbps line rates. We evaluate a server-based and a serverless mode of training. 
In server-based training, the nodes involved in the training, referred to as {\it workers}, 
exchange gradients with a dedicated server. Here, the transmissions to the server 
create a bottleneck. Serverless training avoids this pitfall by exchanging gradients in a distributed fashion. The measurement experiments in this paper use 
Ring Allreduce~\cite{ringallreduce}, a widely used technique for serverless gradient aggregation, 
where nodes are arranged in a logical ring, and gradients are disseminated along the ring. 

A challenge for assessing the short-term burstiness of machine learning traffic 
is the lack of suitable metrics for quantifying and comparing burstiness properties. Metrics that are computed over an entire traffic trace, such as the ratio 
of the peak rate to the average rate do not capture the time scale 
at which the maximum traffic rate is observed. Likewise,  percentiles, cumulative distribution functions, or other statistics generally fail to capture isolated phenomena, such as a singular large burst. In this study, we propose a set of  metrics that can quantify the degree of traffic burstiness at different time scales. We show that these metrics are effective not only at extracting whether a traffic trace contains isolated 
large  bursts, but, moreover, identify the time scale at which such bursts occur. The metrics are formulated using concepts from network calculus \cite{Book-LeBoudec,Book-Bouillard}, a theory for a deterministic analysis of networks. 

Our analysis of the training traffic of ResNet-50 and the application of our metrics show that the traffic pattern of model training follows an on-off pattern, where lengthy time periods with 
very few transmissions alternate with time periods that experience a sequence of burst transmissions.  
The on-off pattern reflects the operation of backpropagation 
in DNN training. In an on-phase, the number of observed traffic bursts relates to 
the number of convolution layers of the DNN model. The sizes of 
traffic bursts relate to the number of parameters in  the layers, ranging from a few thousand to several million in ResNet-50. 

We observe that the short-term burstiness of transmissions of workers is extraordinary, with a peak-to-mean ratio as 
high as  60:1 at time scales as long as 5~ms. At the same time, 
for server-based training, we find that the short-term  
burstiness created by multiple workers does not exceed that of a single worker. We are able to trace this back to 
a coordination of transmissions that prevents multiple workers from 
sending bulk transfers concurrently to the server. 

To evaluate how concurrent training applications 
can jointly create congestion, we introduce the notion of 
{\it burstiness potential}. Burstiness potential is a metric 
that evaluates the degree to which aggregated traffic sources 
approach a worst-case alignment.  
We present  a simulation for such a worst-case alignment, which shows the challenge that distributed ML traffic 
poses for congestion control in data center networks, specifically, the 
transmission of large bursts that follow an extended time period with few or no transmissions. 

While our traffic analysis is made for a particular DNN model  and computing environment, the results in this paper provide reference points for future studies of machine learning traffic. Distributed training of convolutional neural network other than ResNet-50 should be expected to yield qualitatively similar results. The involvement of GPUs in training  accelerates computation  and  shift the bottleneck of training from computation toward communication. This will increase the traffic burstiness of the gradient exchanges, while preserving the overall traffic volume. On the other hand,  communication patterns and overheads may vary considerably across different models and different types of training. The degree to which a different gradient aggregation method, such as AlltoAll, Tree Allreduce and other 
methods~\cite{alltoall}, different ML applications, such as natural language processing~\cite{GPT, BERT}, or different 
learning approaches, such as reinforcement learning \cite{dqn}, impact the communication volume and burstiness patterns remains an open question and awaits further investigation.

\section{Metrics for Traffic Burstiness}
\label{sec:metrics}

Traffic burstiness is often used as a qualitative criterion to describe
that network traffic is highly time-variable. 
For quantifying the degree of traffic burstiness,  a frequently used metric is the {\it 
peak-to-mean ratio} (or {\it peak-to-average ratio}), which refers 
to the ratio of the peak rate  to the average rate of traffic. Such a scalar 
metric, however, is insufficient for studying short-term bursts since it does not 
account for  different degrees of traffic burstiness that exist at 
different time scales. 
We address this shortcoming by  defining  metrics that provide the 
burstiness of traffic over any given time interval. 
The metrics are defined within the context of the network 
calculus~\cite{Book-LeBoudec,Book-Bouillard}. We will show that 
some of the metrics can be used to isolate the presence of large traffic bursts, 
even if a traffic burst is a singular event.

\subsection{Network Calculus Background}

Essentially, network calculus is a method for characterizing the input-output relationship of 
traffic at a network element, where a network element may represent one or a 
collection of switches, traffic control algorithms, or transmission media. We consider a discrete-time clock, where time has 
values $t = 0, 1,2, \ldots$. Arriving traffic to and departing from a network element is described by functions $A$ and $D$, such that 
$A(t)$ and $D(t)$, respectively, represent the cumulative arrivals and departures at a network element until time $t$,  with $D(t)\le A(t)$.  The service at a network element  is described by another function $S$,  referred to as the element's service curve. 
Like functions $A$ and $D$, a service curve $S$ is non-negative and non-decreasing. For technical reasons we assume $A(t) = D(t) = S(t) =0$ for $t=0$. 
The backlog at a network element $B$ is a function that describes the traffic arrivals that 
have not departed, that is, $B(t) = A(t)-D(t)$. 
The maximum backlog, $B_{\rm max}$, is then determined by 
$B_{\rm max}(t) = \max_{t\ge 0}\{A(t)-D(t)\}$.

A network element transforms  arriving traffic to generate the traffic pattern that departs the network element. With network calculus, the interaction between arriving traffic and a network element is described in terms of a min-plus convolution of the arrival function and the service curve of the network element. The min-plus convolution  of two functions $f$ and $g$, denoted by $f \conv g$, is defined by $f\conv g (t) = \min_{0 \le s \le t} \{f (s) + g (t-s)\}$.  If $D (t) = A \conv S(t)$ holds for all times~$t$, the service curve $S$ is referred to as an exact service curve. If the inequality $D (t) \ge A \conv S(t)$ holds for all~$t$, we speak of a lower (or simple) service curve. 

An egress port at a switch is modeled as a 
work-conserving buffered link with transmission rate $r >0$ 
that offers to arriving traffic an exact service curve 
$S(t) = \max\{rt, 0\}$. 

Since $A(t) - A(s)$ is the traffic arriving in the time interval between $s$ and $t$, arrival functions  can  characterize traffic burstiness for any given time interval. 
However, we are not interested in what happens in any specific time interval, 
but in what happens in an arbitrary time interval of a given length. For instance, we may want to know whether the maximum 
amount of traffic that arrives at a switch port in any interval of length 50\,$\mu$s can exceed $10$\,MB.
This calls for a time-invariant characterization of 
traffic. Such a characterization describes traffic in a time interval without regard to the  
location of the time interval. 
This leads us to the concept of an {\it arrival curve}, 
which  specifies an upper  bound on the amount of traffic in  a given time interval.
A function $E$ is called an {\it arrival curve} for 
an arrival function $A$ if $E (s) \geq A (t+s) - A (t)$ for all $s,t \geq 0$. 
In other words, $E (s)$ is an upper bound on the arrivals in any time interval of length~$s$. 
All subsequently introduced traffic burstiness metrics are based on the computation of a specific arrival curve, referred to as {\it burstiness curve}.

\vspace{5pt}

\subsection{Burstiness Metrics}
\label{subsec:burst-metrics-def}
For a given traffic trace, the cumulative arrival function $A(t)$ simply is the total 
number of bytes that have arrived until time $t$. 

\bigskip
\noindent{\bf Burstiness Curve. } We define the burstiness curve $\E_A$ of an arrival function $A$ as the function 
that describes the maximum amount of traffic that arrives in any time interval. 
The burstiness curve $\E_A$ is defined by 
\begin{align*}
\E_A (\tau)  = \max_{t \ge 0} \{ A(t+\tau) - A(t)\} \, . 
\end{align*}
The burstiness curve is the best possible arrival 
curve for a given arrival function, in the sense that any function smaller than $\E_A $ is not 
an arrival curve for $A$.  The burstiness curve was first introduced in \cite{Wrege1}, 
where it is referred to as {\it empirical envelope}.
The burstiness curve can be equivalently expressed as  
\begin{align*}
\E_A (\tau) & = A \deconv A (\tau)  , 
\end{align*}
where the `$\deconv$' operation is the min-plus deconvolution, which, for two functions $f$ and $g$, 
is defined as $f\deconv g (t) = \max_{s \ge 0} \{f (t+s) - g (s)\}$. 
%Since the deconvolution of a non-negative, non-decreasing, and causal function with itself is subadditive,\footnote{A function $f$ is {\it subadditive} if $f (t + s) \leq f (t) + f (s)$  
%for all $s, t \in \RR$.} the second formulation lets us conclude that the 
%burstiness curve is a subadditive function. 

\bigskip
\noindent{\bf Peak-to-mean ratio function:} For a time interval of length~$\tau$, the 
peak-to-mean ratio ${\rm PtM} (\tau)$ is computed as 
\[
{\rm PtM} (\tau) = \frac{\E_A (\tau)}{\bar{\lambda}\tau} \, , 
\]
where $\bar{\lambda}$ is the average rate of the traffic trace. 
An advantage of the peak-to-mean ratio function over the burstiness curve is that it is unitless on 
the y-axis. This enables us to compare the burstiness of different traffic traces independent  of the volume and rate of traffic. 

\bigskip
\noindent{\bf Maximum backlog function:} 
An alternative method for measuring traffic burstiness is to evaluate the backlog inflicted by traffic at a switch. 
Note that the presence of microbursts is indicated by a high buffer occupancy 
at a switch port, which is -- on average -- only lightly loaded. We next 
provide a metric that determines whether a given traffic  may cause a large backlog at a lightly 
loaded port.  The metric is based on tracking the maximum backlog 
at a switch port where the traffic of interest can utilize a certain portion of the link rate, referred to as {\it available link rate}. (The available link rate  can be viewed as the available bandwidth \cite{availableBW} or, alternatively, the fair share of the total link rate under a fair queueing method.)

Let us denote the maximum backlog for an available link rate $r$ by $B_{\rm max} (r)$. Using that a work-conserving 
link with rate $r$ has an exact service curve $S(t) = \max\{rt, 0\}$, we obtain  
\[
B_{\rm max} (r) = \max_{t \ge 0} \{A (t) - A \conv S (t)\} \, . 
\]
A similar metric has been proposed in 
\cite{Low1,Low2,Low3} for fluid flow traffic.\footnote{In~\cite{Low1,Low2,Low3}, the maximum backlog 
as a function of the link rate is referred to as {\it burstiness curve}. We justify our change of 
terminology by the fact that  $\E_A$ is more directly related to 
the actual burstiness of the traffic.} Specifically, if $m(\tau)$ describes the 
instantaneous arrival rate at time $\tau$ for traffic that is active in the time interval 
$[0,T]$, the maximum backlog in  \cite{Low1,Low2,Low3} 
is computed as 
$B_{\rm max} (r) = \max_{0 \le s \le t \le T} \int_{s}^{t} [m(\tau) - r]\ d\tau$. 

The following lemma %, which is proven in Appendix~\ref{sec:app-lemma}, 
allows us to relate the maximum backlog to the burstiness curve. 

\begin{lemma} \label{lemma:Bmax}
The maximum backlog of traffic with an arrival function $A$  at a network element  with exact service curve $S$ satisfies $B_{\rm max} = \E_A \deconv S(0)$. 
\end{lemma}
\begin{proof}
\begin{align*}
B_{\rm max} & = \max_{s \ge 0} \{A (t) - D (t) \} \\
& =  \max_{t \ge 0} \{A (t) - A \conv S (t) \}\\
& =  \max_{t \ge 0} \Bigl\{A (t) - \min_{0 \le s \le t} \{ A (t-s) + S(s) \} \Bigr\} \\
& =  \max_{t \ge 0} \max_{0 \le s \le t} \Bigl\{ A (t) -  A (t-s) - S(s) \Bigr\} \\
& =  \max_{s \ge 0} \Bigl\{ \max_{t \ge s} \{ A (t) -  A (t-s)\}  - S(s) \Bigr\} \\
& = \max_{s \ge 0} \{ \E_A(s) - S(s) \} \\
& = \E_A \deconv S(0) \, . 
\end{align*}
The second line uses that $S$ is an exact service curve. The fifth line results from switching the order of the maxima. Then, the inner maximum is equal to the burstiness curve, which yields the result.
\end{proof}

A classical result of the network calculus states that the backlog for traffic with arrival curve $E$ at a network element with lower service curve $S$ is bounded by $B(t) \le E \deconv S(0)$. Lemma~\ref{lemma:Bmax} states that the bound is tight when the service curve is exact and the arrival curve is tight.

With the lemma we obtain an expression of the maximum backlog function $B_{\rm max} (r)$ in terms of the burstiness curve. 
\begin{corollary}
Consider traffic with arrival function $A$ at a work-conserving link with rate $r$. Then 
\begin{enumerate} 
\item $\displaystyle B_{\rm max} (r)  = \max_{\tau \ge 0} \{ \E_A (\tau) - r\,\tau\}$. 
\item $B_{\rm max} (r)$ is convex.
\end{enumerate}
\end{corollary}
\begin{proof}
The first claim follows from  Lemma~\ref{lemma:Bmax} since $S(t) = \max \{rt, 0\}$ is an 
exact service curve for a work-conserving link with rate $r$. For the second claim, 
note that, for a fixed value of $\tau$, the term inside the maximum of  $B_{\rm max}$ (in the first claim) is a linear decreasing function for $r >0$, and hence convex.  The claim follows since the maximum of convex functions is also convex.  
\end{proof}
%We point out that the expression for $B_{\rm max} (r)$ can be related to a version of the  Legendre-Fenchel conjugate, given by ${\mathcal L}_f (r) = \max_{t\ge 0} \{r t - f(t)\}$ for a function $f$. With this, the maximum backlog function is given by ${\mathcal L}_{-\E_A} (-r)$.  

When plotting $B_{\rm max} (r)$ for increasing values of $r$, we use the function $r/\bar{\lambda}$ 
on the x-axis.  This normalization with the average traffic rate allows us to relate the $x$-axis 
to the utilization $U$, which is defined as  
\[
U  = \frac{\text{average traffic rate}}{\text{available link rate}} = \frac{\bar{\lambda}}{r} \, . 
\]
Hence, $r/\bar{\lambda}$ corresponds to the inverse utilization $1/U$. The advantage of using the utilization on the x-axis is that it is a unitless measure, which allows us to compare the burstiness of traffic with different characteristics. 

\bigskip
\noindent{\bf Interval-specific maximum backlog:}  
By plotting the maximum backlog for a given available rate $r$ for 
different interval lengths, we can illustrate at which time scales the maximum backlog is attained. 
This metric is simply given by 
\[
B_{\rm max} (\tau ; r)  = \max \{ \E_A (\tau) - r \, \tau \, , 0 \} \, .  
\]

%If we 
%express $U$  as a percentage, then we obtain the following mapping:
%
%\begin{center}
%\begin{tabular}{l|c  c  c  c  c }
%$1/U$  &  1.05  & 2   & 5 & 10 & 20 \\
%\hline 
%$U$ (in \%) & $\approx 95$\% & 50\%  & 20\% & 10\% & 5\% \\
%\end{tabular}
%\end{center}
%

%%%%------------------------------------------------------------------
\begin{figure}[!t]
\centering 

\subfloat[Traffic \underline{with} rate spike.]{\includegraphics[width=0.49\textwidth]{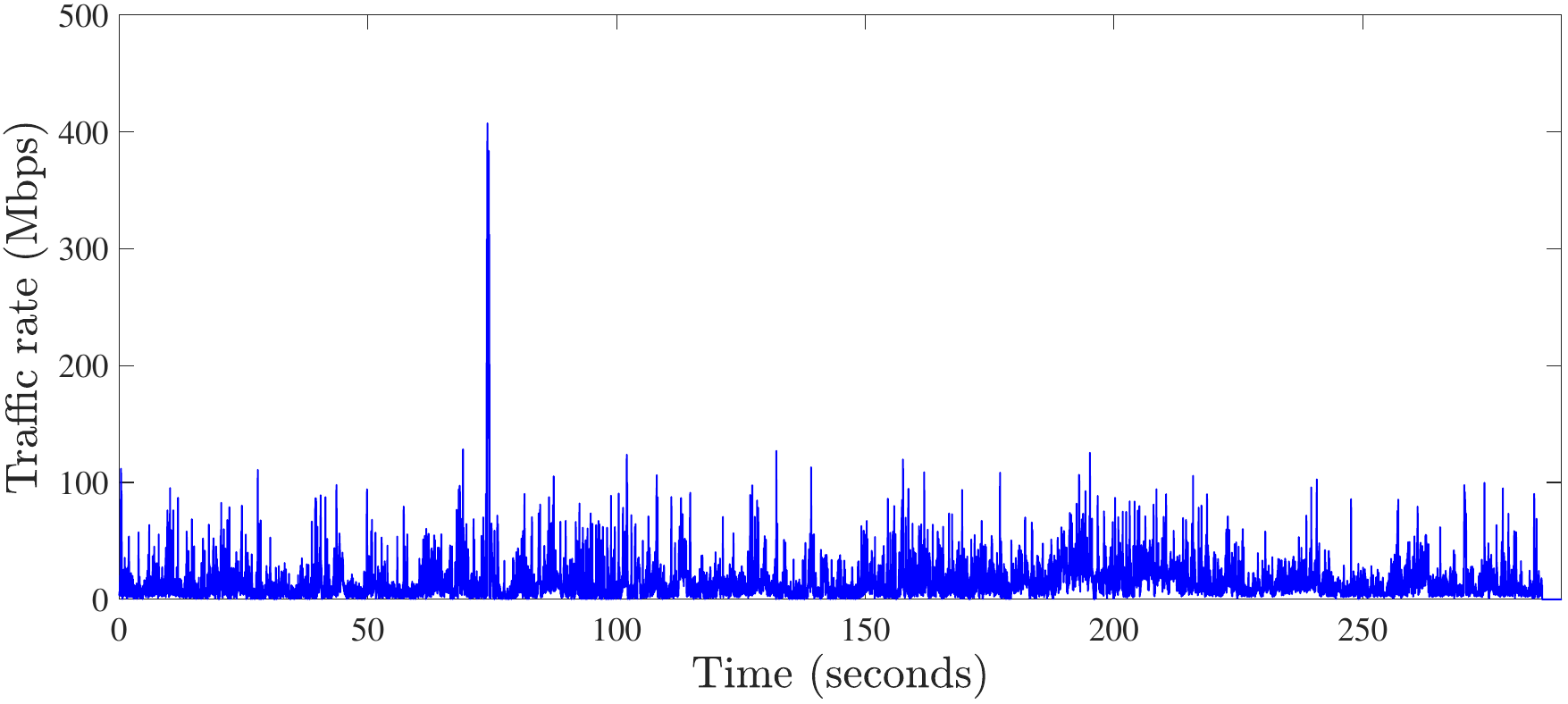}\label{fig:Wisc-univ1-pt11-TrafficRate}}
 %(averaging interval: 20~ms)
\hspace{0.2cm}
\subfloat[Traffic \underline{without} rate spike.]{\includegraphics[width=0.49\textwidth]{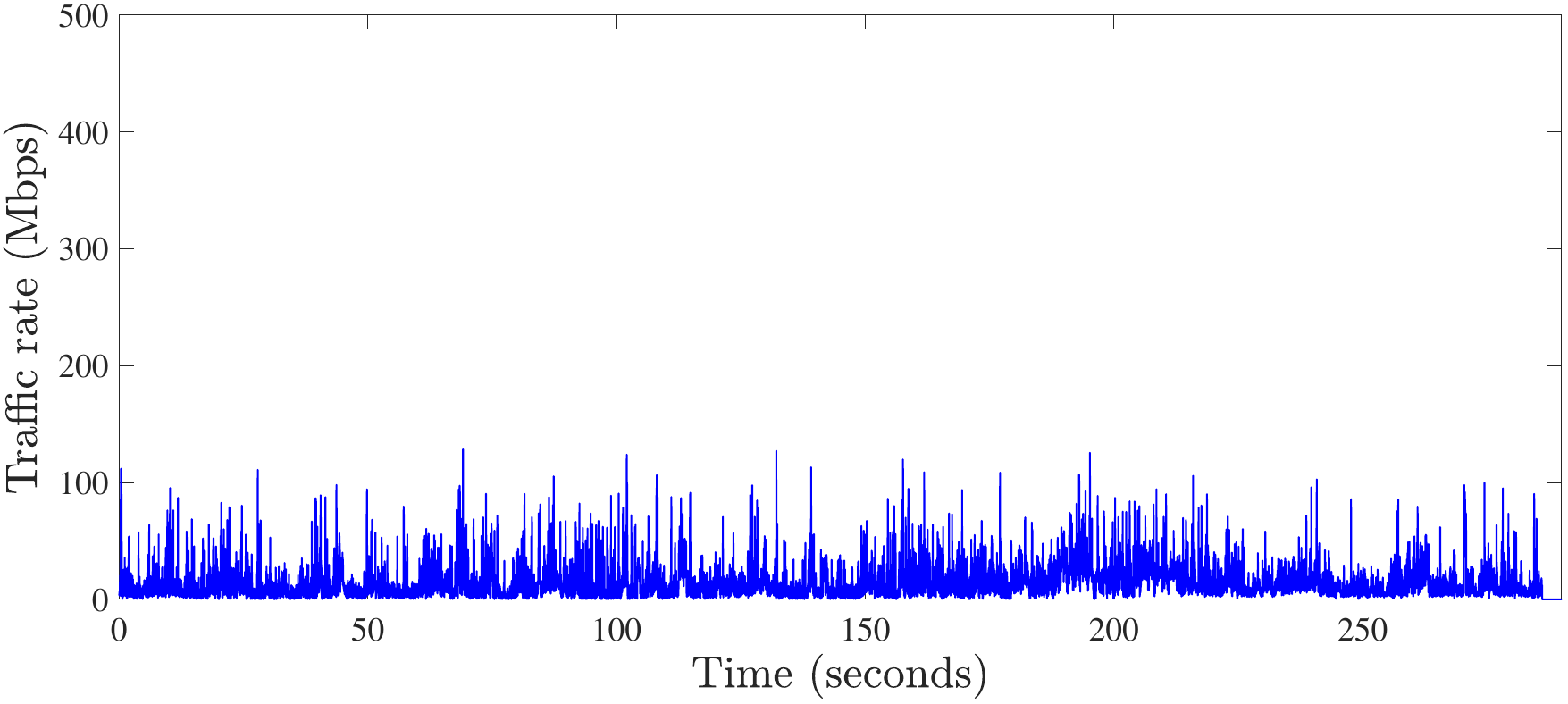}\label{fig:Wisc-univ1-pt11-mod-TrafficRate}}

\vspace{-5pt}
\caption{Traffic rates of DCN traffic trace \underline{with} and \underline{without} large rate spike.} \label{fig:univ1-pt2-T2}
%\end{figure}
%%%%%------------------------------------------------------------------
%\begin{figure}[!t]
\centering 
\subfloat[Burstiness curve  and average 
traffic.]{\includegraphics[width=0.24\textwidth]{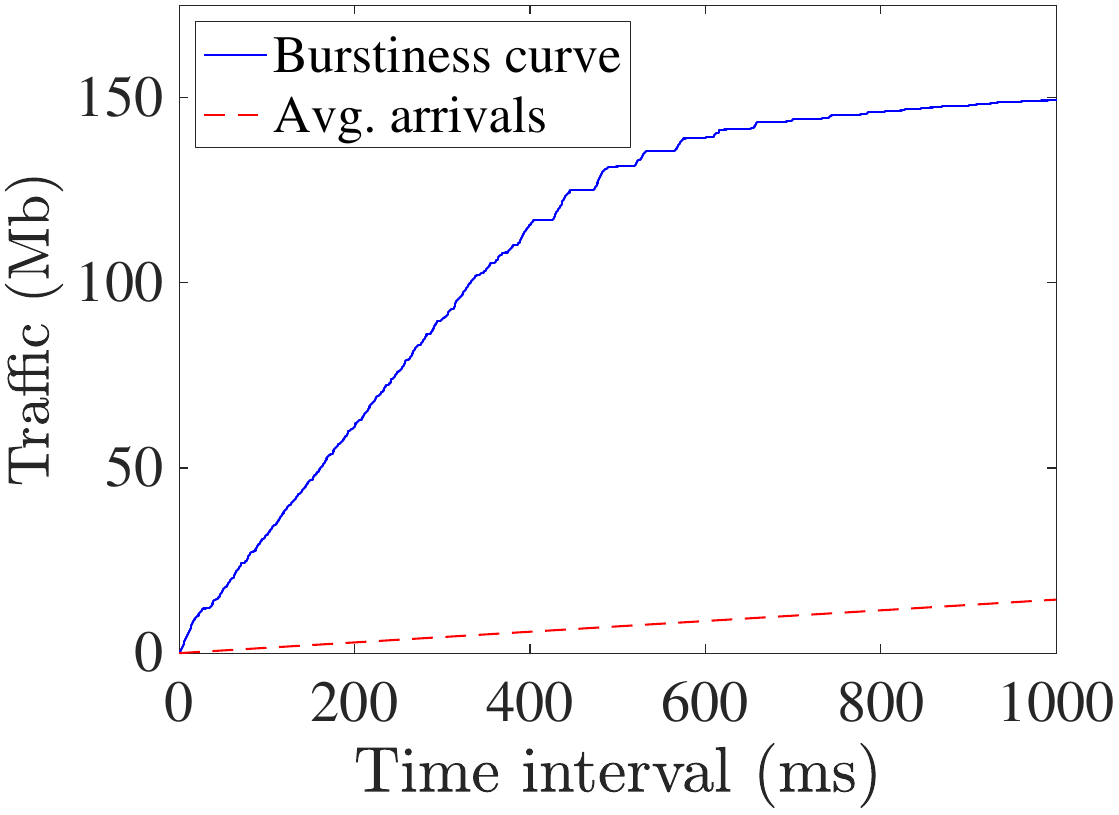}\label{fig:Wisc-univ1-pt11-burstiness}}
%
%\hspace{1cm}
%
\subfloat[Peak-to-mean 
ratio.]{\includegraphics[width=0.24\textwidth]{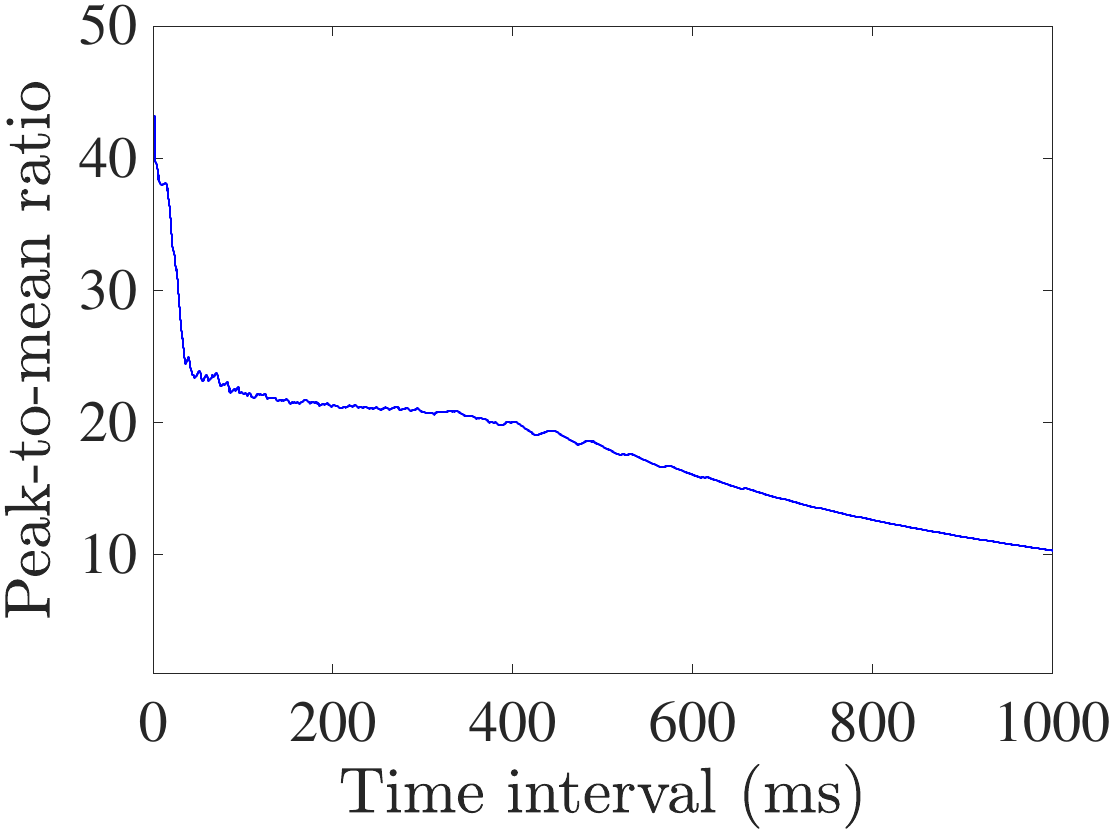}\label{fig:Wisc-univ1-pt11-peak-to-mean}}
%
%\vspace{-5pt}
%\centering 
%
\subfloat[Maximum backlog.]{\includegraphics[width=0.24\textwidth]{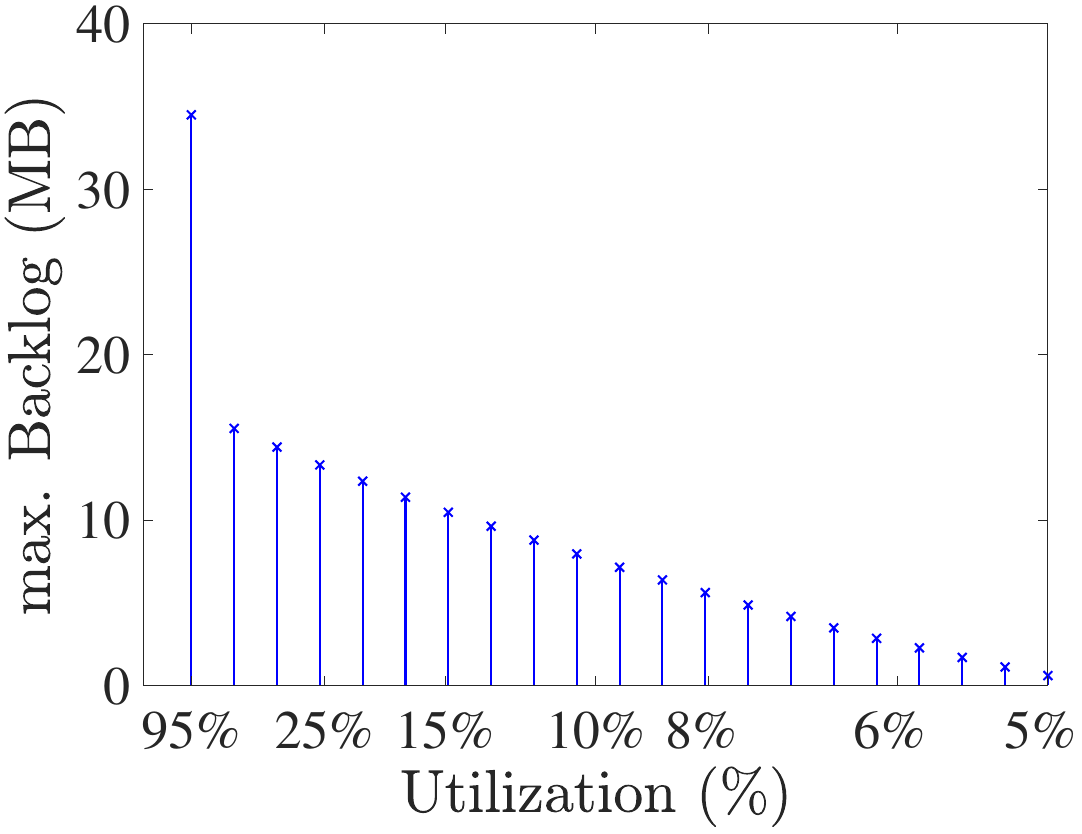}\label{fig:Wisc-univ1-pt11-maxBacklog}}
\subfloat[Interval max. backlog.]{\includegraphics[width=0.24\textwidth]{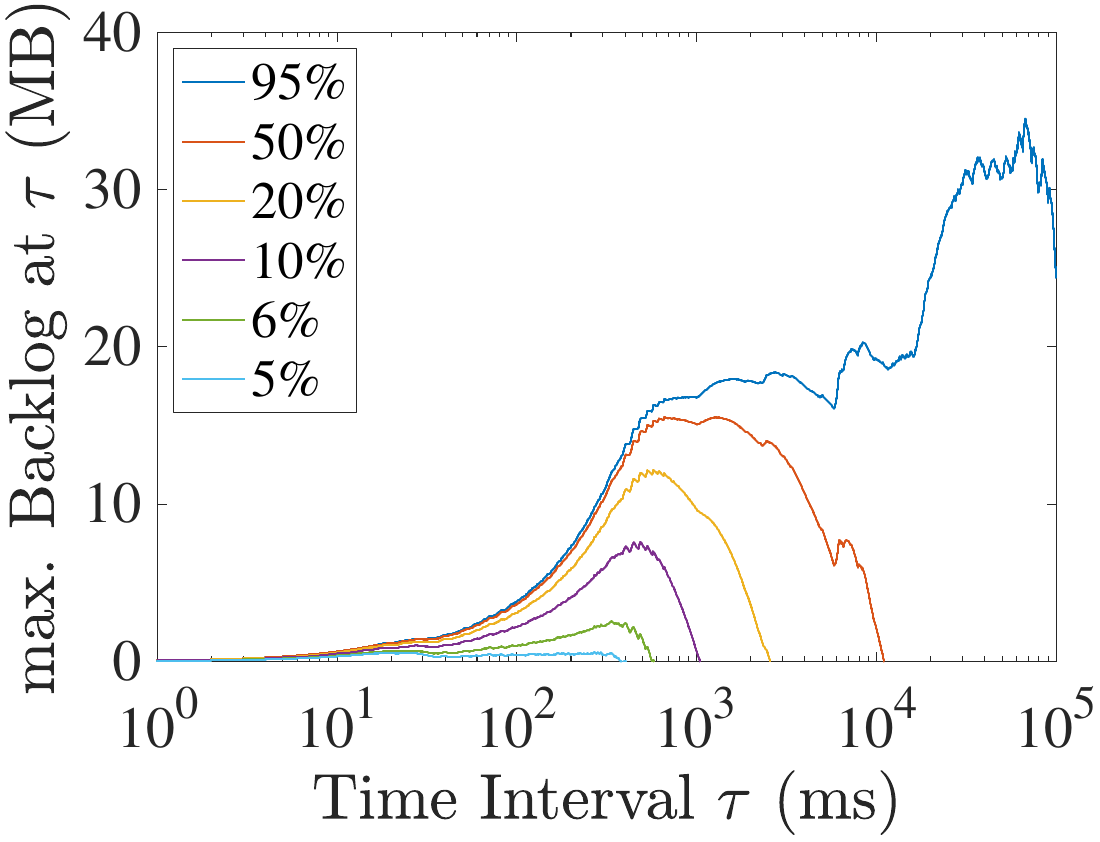}\label{fig:Wisc-univ1-pt11-maxBacklog1-interval}}

% Granularity of metrics: $\Delta = 1 ms$
\vspace{-5pt}
\caption{Burstiness metrics of DCN traffic trace \underline{with} large rate spike.} \label{fig:univ1-pt2-T2}
%\end{figure}
%
%
%
%%%------------------------------------------------------------------
%
%%%%%------------------------------------------------------------------
%\begin{figure}[!t]
\centering 
\subfloat[Burstiness curve and average 
traffic.]{\includegraphics[width=0.24\textwidth]{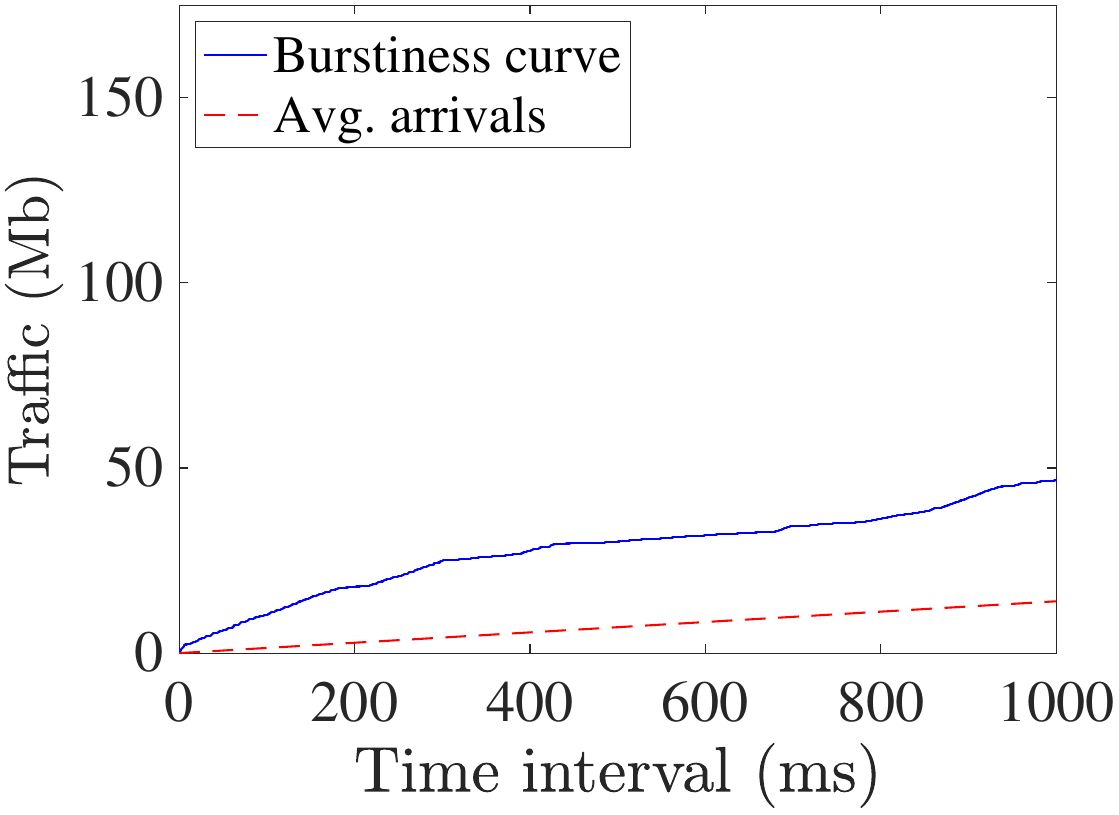}\label{fig:Wisc-univ1-pt11-mod-burstiness}}
%
%\hspace{1cm}
%
\subfloat[Peak-to-mean 
ratio.]{\includegraphics[width=0.24\textwidth]{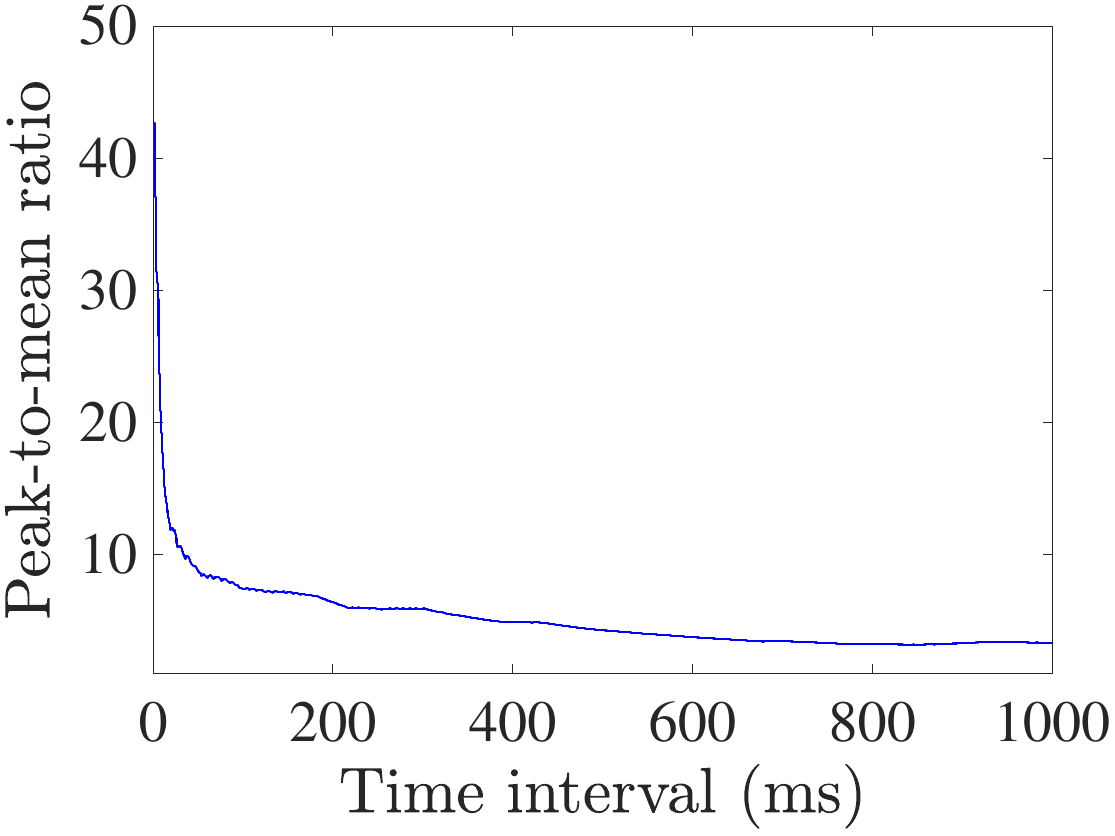}\label{fig:Wisc-univ1-pt11-mod-peak-to-mean}}
%
%\vspace{-5pt}
%\centering 
%%
\subfloat[Maximum backlog.]{\includegraphics[width=0.24\textwidth]{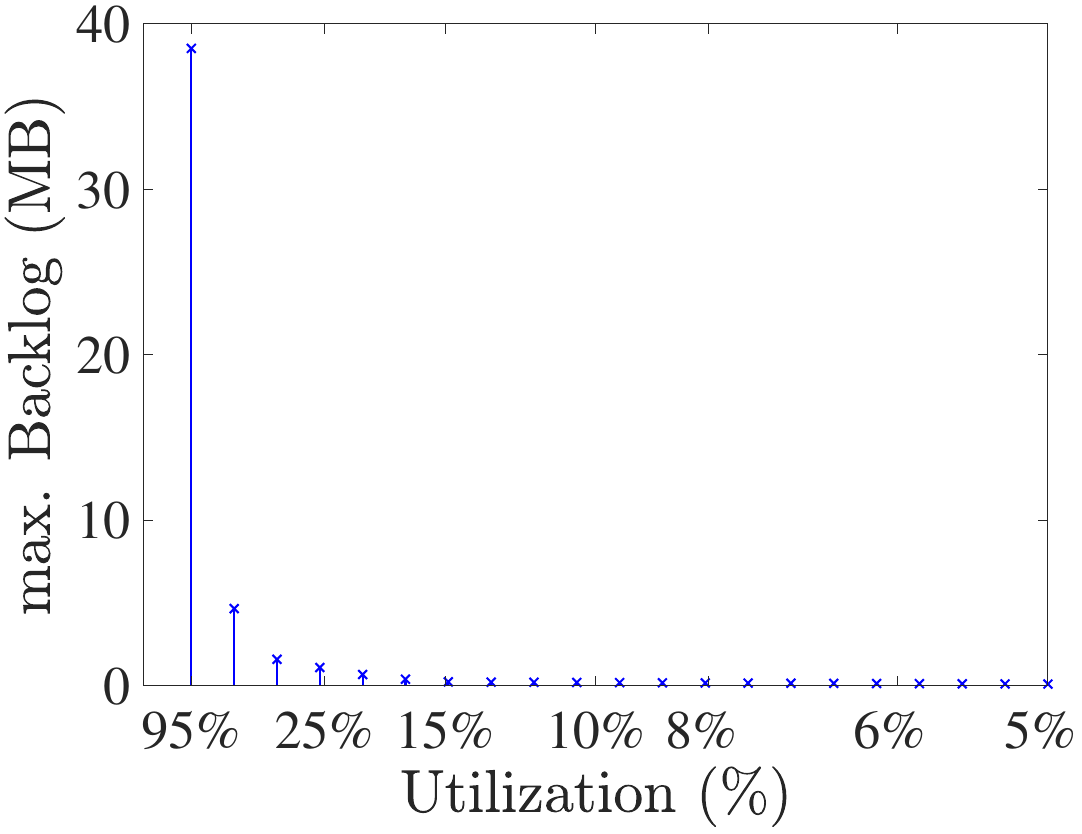}\label{fig:Wisc-univ1-pt11-mod-maxBacklog}}
\subfloat[Interval max. backlog.]{\includegraphics[width=0.24\textwidth]{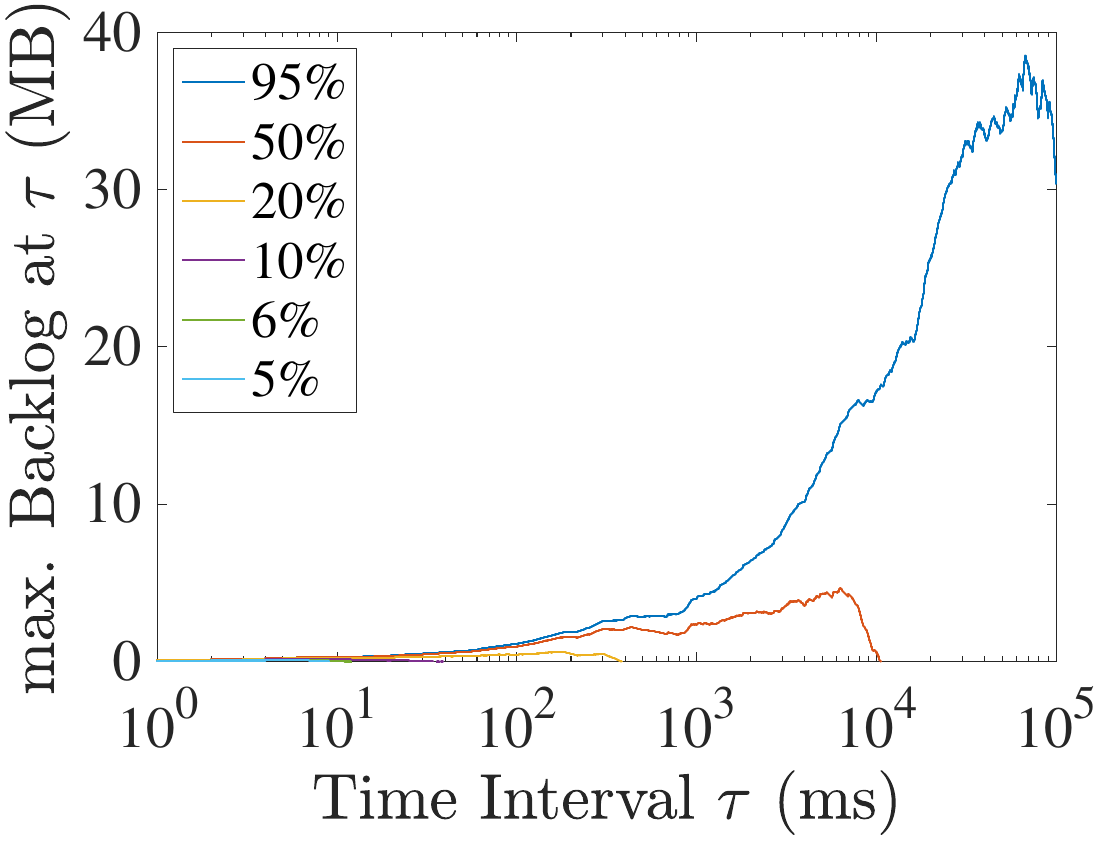}\label{fig:Wisc-univ1-pt11-mod-maxBacklog1-interval}}

% Granularity of metrics: $\Delta = 1 ms$
\vspace{-5pt}
\caption{Burstiness metrics of manipulated DCN traffic trace \underline{without} rate spike (spike removed).} \label{fig:univ1-pt11-mod}
\end{figure}
%%------------------------------------------------------------------

\subsection{Examples}

We next demonstrate how the defined  metrics extract information about the presence of 
large bursts. For demonstration, we consider a DCN measurement of a  1\,Gbps switch port in an  operational data center~\cite{dcn-trace-1}, with traffic traces available 
in \cite{wisc-traces}. 
The selected trace\footnote{We use the trace with label `univ1\_pt11' from~\cite{wisc-traces}.} 
has roughly $10^6$ packets sent over a duration of close to 6~minutes with an 
average traffic rate of around 14~Mbps. 
Fig.~\ref{fig:Wisc-univ1-pt11-TrafficRate} depicts the traffic rates of 
the trace, computed for time windows of 20~ms, which 
clearly shows one large rate spike. 

To create a reference for comparison, we manipulate 
this trace by deleting the rate spike. This is done by substituting 
the captured data in the time interval that contains the large burst (we use interval [73.8,74.5]\,s) with data that is copied from a different interval  
with the same duration (interval [173.8,174.5]\,s). 
Fig.~\ref{fig:Wisc-univ1-pt11-mod-TrafficRate} shows the traffic rates 
of the manipulated trace.   

A comparison of the burstiness metrics of the original and manipulated traces, given in 
Figs.~\ref{fig:univ1-pt2-T2} and~\ref{fig:univ1-pt11-mod}  
shows that the presence of the large burst in the original trace results in noticeably different burstiness metrics. 
Figs.~\ref{fig:Wisc-univ1-pt11-burstiness},~\ref{fig:Wisc-univ1-pt11-peak-to-mean} and 
Figs.~\ref{fig:Wisc-univ1-pt11-mod-burstiness},~\ref{fig:Wisc-univ1-pt11-mod-peak-to-mean} 
show the burstiness curves and the peak-to-mean ratio for time intervals up to 1\,s. We  observe that the presence of the rate spike results in a much higher peak-to-mean ratio. 
The maximum backlog plots in Figs.~\ref{fig:Wisc-univ1-pt11-maxBacklog}, ~\ref{fig:Wisc-univ1-pt11-mod-maxBacklog} provide the clearest indication of the presence of microbursts. 
The backlog 
in Fig.~\ref{fig:Wisc-univ1-pt11-maxBacklog} is 
non-negligible even when the link utilization is as low as 5\%.
This indicates that 
even at very low loads, the traffic trace creates a considerable backlog, pointing to the presence of a microburst. Also, for a link utilization below 50\%,  the maximum backlog 
decreases linearly with increasing link rates. This linear decrease suggests 
that the maximum backlog at different link rates occurs for 
burst arrival events, and that the observed decline of the 
maximum backlog is merely due to increasing the available link rate. 
In contrast, without the rate spike 
(Fig.~\ref{fig:Wisc-univ1-pt11-mod-maxBacklog}) 
there is a noticeable backlog only at high 
utilizations.\footnote{For a utilization of 95\%, the maximum backlog of the manipulated trace is larger 
than that of the original trace, even though the original trace has more data. This is due to the removal of the rate spike, which reduces the average  
rate of the trace. As a result, the available link rate for a 95\% utilization with the manipulated trace is less than for the original trace, which 
results in a larger maximum backlog.}

Figs.~\ref{fig:Wisc-univ1-pt11-maxBacklog1-interval}, ~\ref{fig:Wisc-univ1-pt11-mod-maxBacklog1-interval} provide information on the time scale 
at which  the largest backlog is attained. In both graphs, the x-axis 
uses a log scale. For high utilizations, we observe that the large backlog occurs only at a time scale >100\,s. Differently, the maximum backlog for utilizations below 10\% occurs on a much smaller time scale. 

We conclude that  a high peak-to-mean ratio of traffic and a noticeable 
backlog at low utilizations of the available rate  are indications of the potential for microburst events. 
The lower the  utilization at which a non-negligible 
backlog can be observed, the larger the bursts relative to the average traffic.

\begin{figure}[t!]
	\centering
	\includegraphics[width=0.7\textwidth]{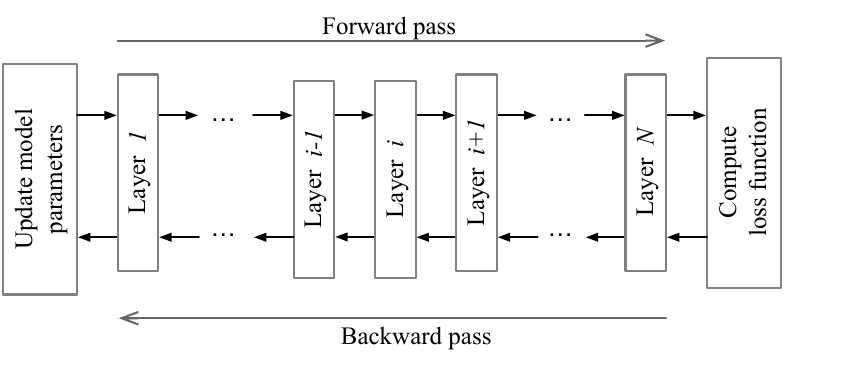}
	\caption{Workflow of DNN model training.}
	\label{fig:DNN-training}
	
	\vspace{-10pt}
\end{figure}

\section{Distributed ML Experiments}
\label{sec:distributedML}

This section describes the setup of the measurement experiments of 
the training of a distributed DNN model and the involved software systems.

\subsection{Distributed Training of DNNs}
\label{sec:distributed-training}

Deep neural networks (DNNs) process training data through a sequence of layers, where the output of one layer is the input to the next layer. DNN training  performs an optimization, often a stochastic gradient descent, in rounds, where  each round proceeds in three steps: (1) a {\it forward pass} through the layers of the DNN to compute a loss (or error) function for a training data set, (2) a backward pass where each layer computes gradients of the loss function, and 
(3) an update the DNN model parameters with the computed gradients. 
The workflow of training is shown in Fig.~\ref{fig:DNN-training}.

To accelerate  the training of DNN models with a large number of parameters or  training data sets, distributed training divides the workload of training across multiple processors, which are referred to as {\it workers}. This is generally done by partitioning the 
training data set and having each worker perform training on one partition (data parallelism). Each worker executes the forward path independently from other workers. On the backward pass, the gradients 
of a layer computed by each worker must be shared with all other workers (in practice, the gradients of the workers are averaged) to obtain the gradients for the complete  training data of the current round. This operation is referred to as {\it Allreduce}. Model parameters are updated after the Allreduce operation is completed. One Allreduce is performed for each layer of the DNN model, and the operations 
are executed in sequence starting at the last layer. 
The network traffic of DNN training is dominated by Allreduce operations on the backward pass, and all other network traffic is negligible in comparison.

We consider two approaches to perform an Allreduce, with and without a server. 
\begin{itemize}
\item {\bf Parameter server/Coordinator: }
Here, workers send their gradients to  a dedicated server, which performs the Allreduce operation. 
The server is called a parameter server if it also updates the parameters 
and sends the results to the workers. 
The server is called a coordinator if it only performs the Allreduce on the gradients without updating the parameters. 
For each layer~$i$, the amount of data transmitted by each worker to the server  and by the server to each worker is given by the 
number of parameters of layer~$i$.

\item {\bf Ring Allreduce: } This is a serverless realization of Allreduce where workers exchange data to neighbors in a logical ring topology \cite{ringallreduce,ringallreduce-baidu}. The parameter vector of each layer is partitioned and equally distributed across all workers. There are two rounds of exchanges, referred to as Reduce-Scatter and Allgather. We refer to Appendix~\ref{sec:primer-ringallreduce} for a review of the Ring Allreduce algorithm.  
The total amount of data transferred during a Ring Allreduce is the same as in a server-based Allreduce. An advantage of Ring Allreduce is 
that multiple workers can transmit concurrently to their neighbors in the logical ring. Also, by dividing the parameters of a layer into smaller chunks, Ring Allreduce reduces the burst sizes of transmissions. 
\end{itemize}

\subsection{Measurement Testbed}
\label{subsec:testbed}

Our measurements are conducted on a testbed network as 
shown in Fig.~\ref{fig:topo-exp}, where four hosts are connected to Huawei CE6866-48s8cq switch with 100~GE interfaces and 64\,MB of shared memory. 
Each host is equipped with an Intel-Xeon-511 2.3\,GHz	CPU with 48~cores, 128\,MB RAM, and a Mellanox	ConnectX-5 NIC. The hosts run Ubuntu~20.04 with Linux kernel~5.15.0-41. Due to the installation of the NIC cards in the hosts only 8 out of 16 lanes of the NIC are connected to the motherboard, which limits the throughput of the NIC cards to 50\,Gbps. 
 At each host,   bidirectional traffic  is captured by an {\it n2disk}~\cite{n2disk} traffic recorder  with nanosecond accuracy. Clocks are synchronized to within a microsecond. 
 
 \begin{figure}[t!]
	\centering
	\includegraphics[width=0.6\textwidth]{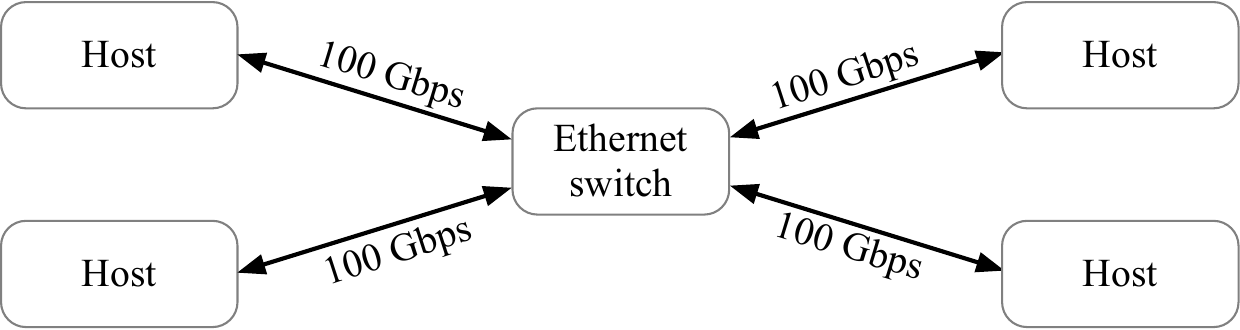}
	\caption{ Testbed network for traffic measurements.}
	\label{fig:topo-exp}
	
	\vspace{-10pt}
\end{figure}
 
The DNN training experiments are performed by TensorFlow (version 2)~\cite{tensorflow} scripts and the open-source Horovod middleware~\cite{horovod}. 
Horovod encapsulates the  Message Passing Interface (MPI) API for parallel computations and the NVIDIA Collective Communications Library (NCCL) for communication between GPUs. 
Since the hosts in the testbed do not use GPUs, communication between workers is based on 
MPI, using the open-source OpenMPI \cite{openmpi} implementation. 
Transport of data between  workers relies on RDMA~\cite{rfc5040} over RoCEv2~\cite{rocev2}. 
In the experiments, the hosts and the switch have both priority flow control~(PFC)~\cite{pfc} and DCQCN congestion control~\cite{DCQCN-Sigcomm2015} enabled. We refer to Appendix~\ref{sec:dcqcn-pfc} for a review of PFC and DCQCN. 

The measurements are done for rounds of training of the ResNet-50 DNN \cite{resnet50} using the ImageNet data set \cite{imagenet}.
ResNet-50 is a model for classification of images with a $224 \times 224$ pixel resolution. 
In a distributed execution, the model requires the transmission of gradients for 54 layers  with a total of  25.5 million parameters, amounting to 102\,MB of data. In each round of training, the model uses a batch size of 32 images.

\section{Server-based Training}
\label{sec:LinearAllreduce}

In this section, we discuss an experiment of distributed ResNet-50 training 
with a server.\footnote{In Horovod, the server acts as a coordinator that  performs an Allreduce operation,  without updating the parameters of a layer.} 
One host acts as the server, and all other hosts are workers, referred to as `Worker1,' `Worker2,' and `Worker3.' We refer to this scenario as {\it Linear-Allreduce}.
The measurements are taken for 10~rounds of training 
over a time period of about 25\,s, during which 
around 6.4 million packets are captured at the server, and 
2.1--2.2 million packets are captured at each of the three workers.

\subsection{Traffic of Linear-Allreduce}
If workers concurrently transmit gradients to the server, there is a risk that the egress port at the switch that leads to the server becomes congested. Since we are interested in 
studying microbursts, we therefore focus on the 
transmission of gradients to the server. 
Over the entire experiment, starting with the first transmissions on the backward pass of the first round until the last transmission on the backward pass of the 10th round, the average traffic rate of each worker is around 340~Mbps. 

Fig.~\ref{fig:linear-allWorkers-10s} presents the transmission rates in Gbps 
for traffic from the workers to the server in a time interval of 10\,s. Each data point presents an average over 1\,ms.  The 10-second interval covers the transmissions over  four rounds of training. We  make the following observations:

%%------------------------------------------------------------------
\begin{figure}[t!]
\centering 
\includegraphics[width=1\columnwidth]{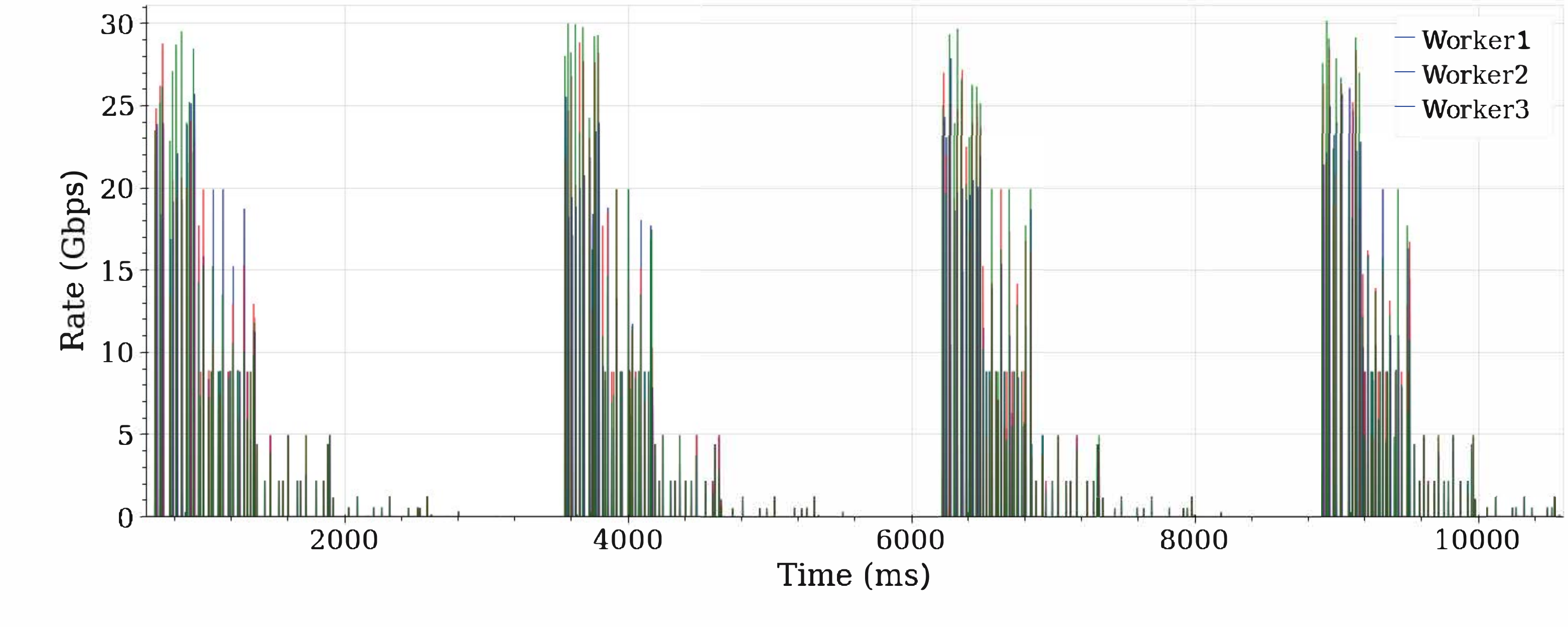}

\caption{Linear-Allreduce: Traffic from all workers to server (10\,s).} 
\label{fig:linear-allWorkers-10s}
\end{figure}
%%------------------------------------------------------------------

\begin{itemize}
\item 
The figure shows four clusters of traffic, where each cluster represents the transmission of gradients on the backward pass of a round. 

\item Each cluster exhibits the same transmission pattern. Initially, the transmission rates are high and then decrease in an almost stepwise fashion.  The  
values of the transmission rates reflect the number of parameters in a layer. In ResNet-50, the number of parameters of the first layers is only a few 
thousand and grows for subsequent layers, exceeding two million for the last layers. Since the transmissions of gradients occur in the reverse order of the layers, the transmission rates in a round are initially high and then decrease. 

\item The transmission rates never exceed 30\,Gbps, even though the line rate is 100\,Gbps. This is partially due to the throughput limit of 50\,Gbps 
and partially due to the fact that burst sizes with ResNet-50 do not exceed~10~MB.  

\item Before each cluster of transmissions in a round there is an extended time period 
with few or no transmissions. This time period corresponds to the forward pass of DNN training. 
The transmissions in this period are negligible in comparison to the traffic observed during the backward pass. 

\end{itemize}

\subsection{Burstiness Metrics}
We evaluate the burstiness metrics from Sec.~\ref{sec:metrics} for the traffic from the workers to the server. 
We compare the burstiness from a single worker to the burstiness of the aggregate traffic from all workers. The metrics are computed for the entire duration of the experiment. 

%%------------------------------------------------------------------
\begin{figure}[t!]
\centering 
\subfloat[Peak-to-mean ratio.]{\includegraphics[width=0.32\textwidth]{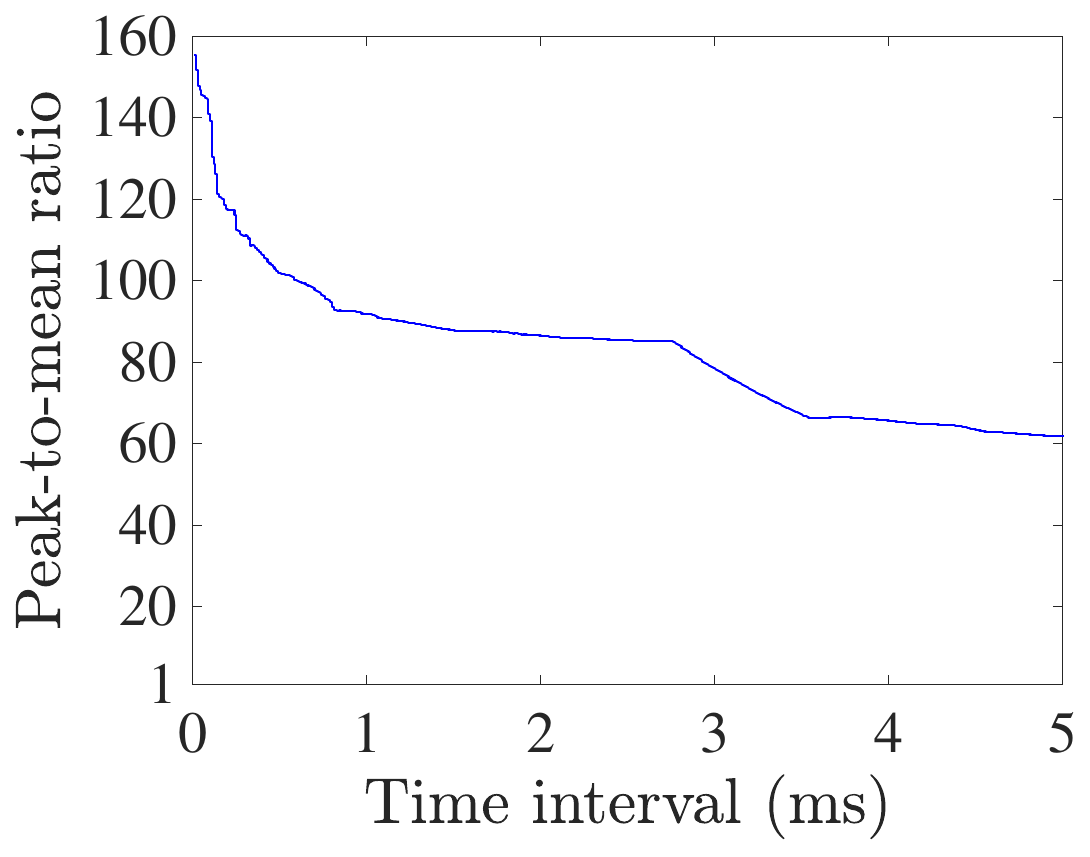}\label{fig:linear-oneworker-peak-to-mean-ratio}}
%
%\hspace{1cm}
%
\subfloat[Maximum backlog.]{\includegraphics[width=0.32\textwidth]{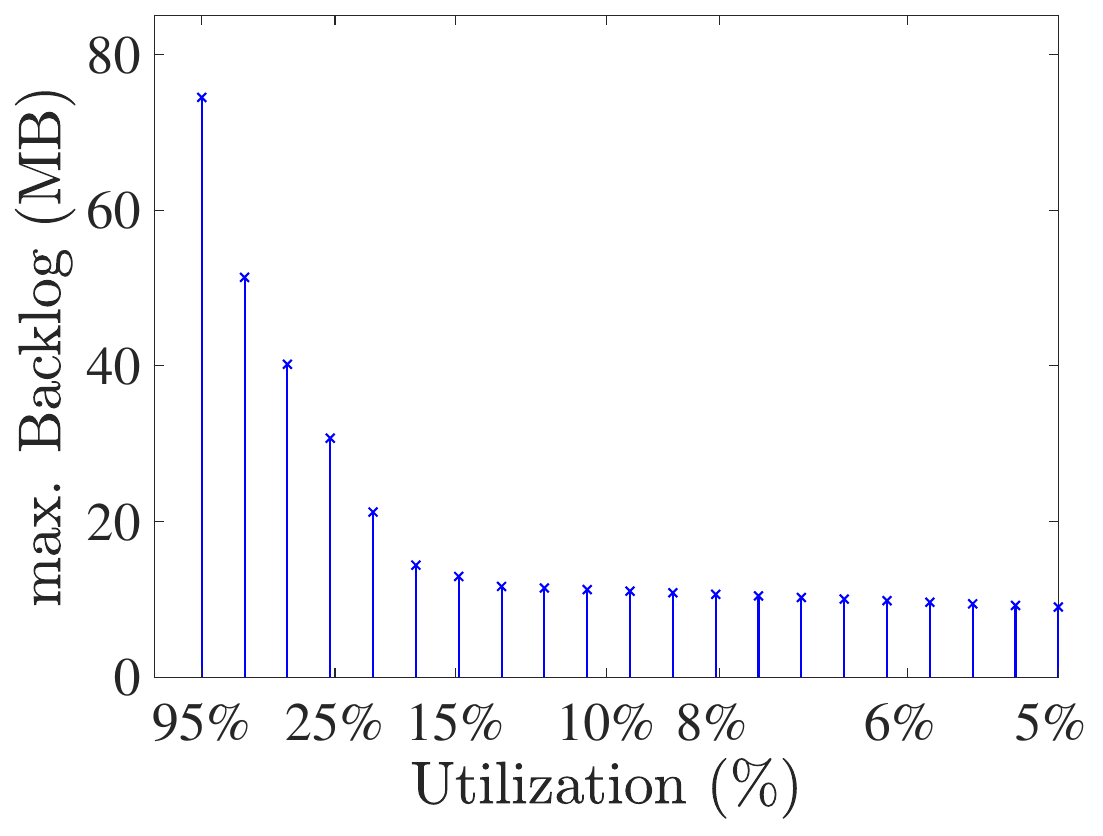}\label{fig:linear-oneworker-maxBacklog}}
\subfloat[Interval max. backlog.]{\includegraphics[width=0.32\textwidth]{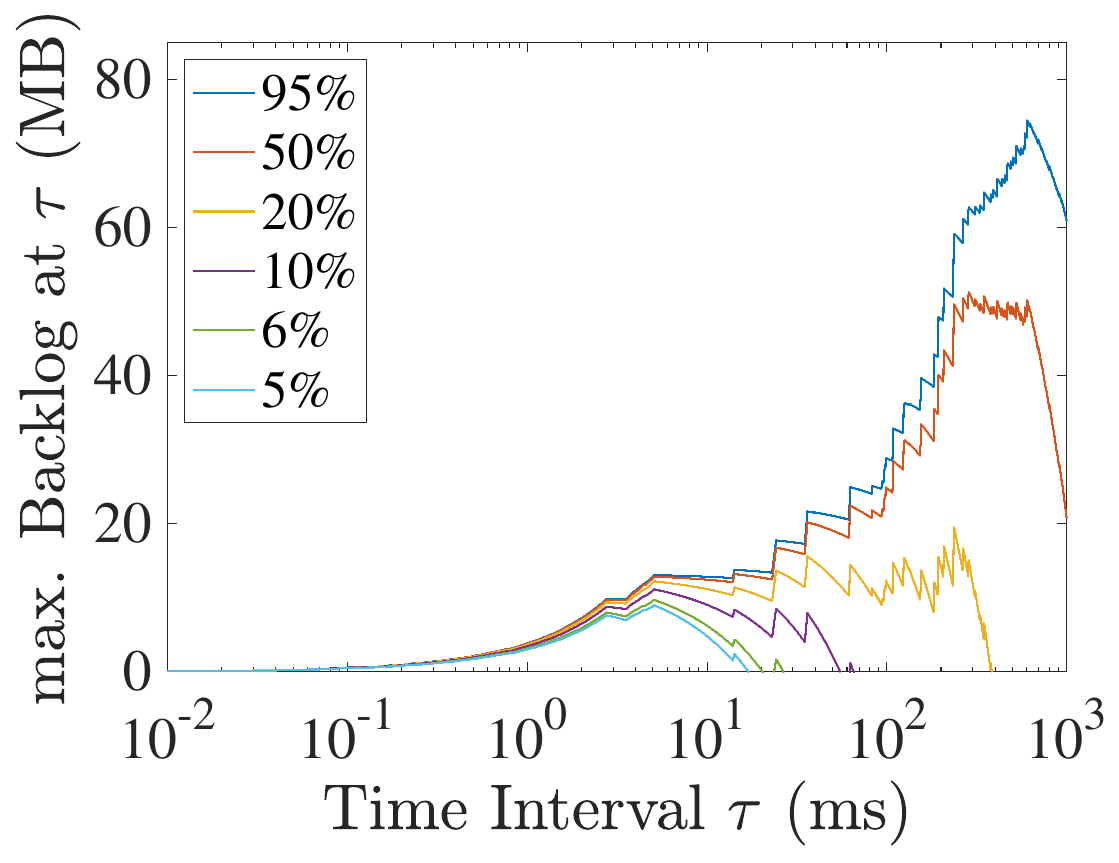}\label{fig:linear-oneworker-maxBacklog1-interval}}

\caption{Linear-Allreduce: Burstiness metrics for traffic \underline{from one worker} (Worker3) to the server.} 
\label{fig:linear-oneworker-burstiness}
%\end{figure}
%%
%\begin{figure}[t!]
\centering 
\subfloat[Peak-to-mean ratio.]{\includegraphics[width=0.32\textwidth]{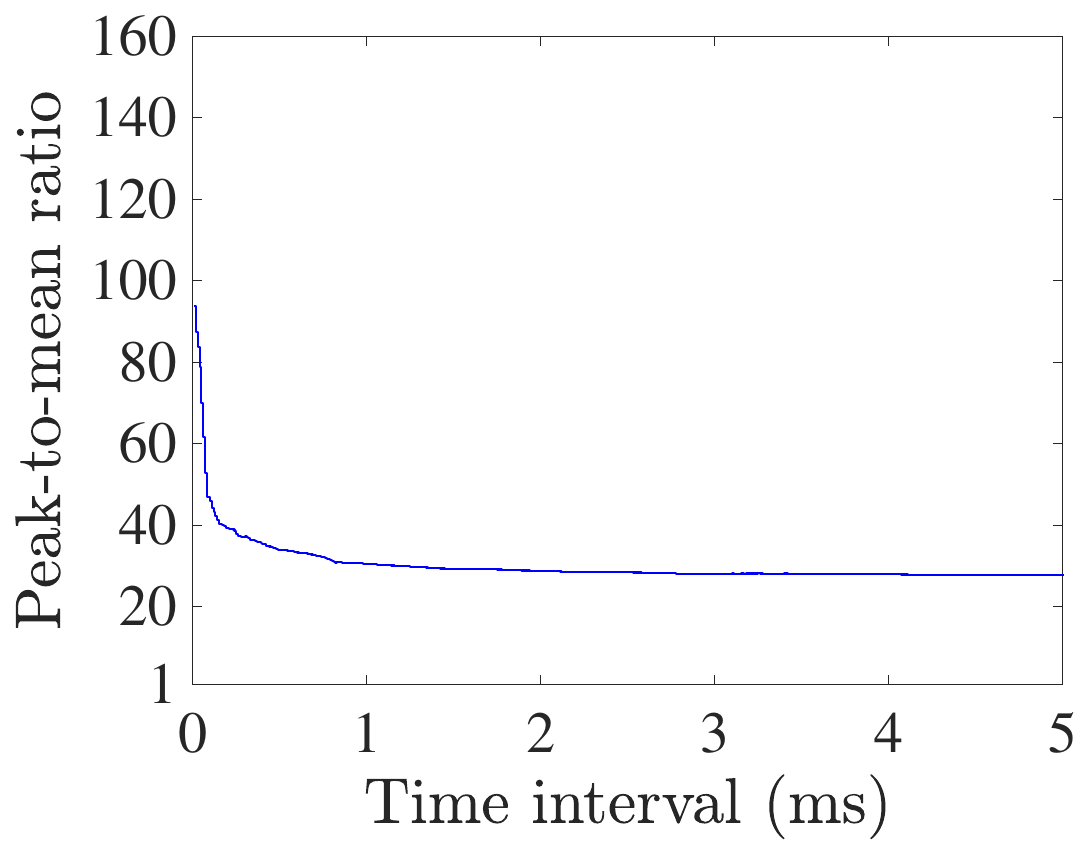}\label{fig:linear-allworkers-peak-to-mean-ratio}}
%
%\hspace{1cm}
%
\subfloat[Maximum backlog.]{\includegraphics[width=0.32\textwidth]{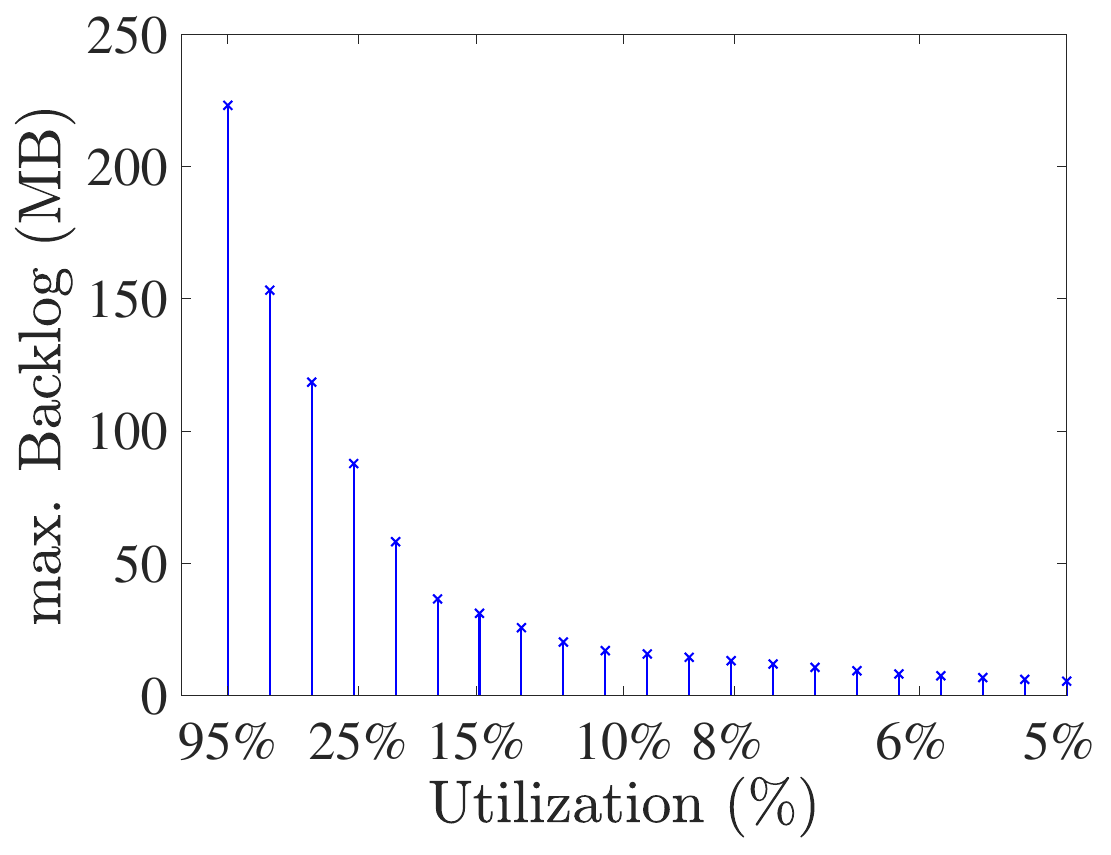}\label{fig:linear-allworkers-maxBacklog}}
\subfloat[Interval max. backlog.]{\includegraphics[width=0.32\textwidth]{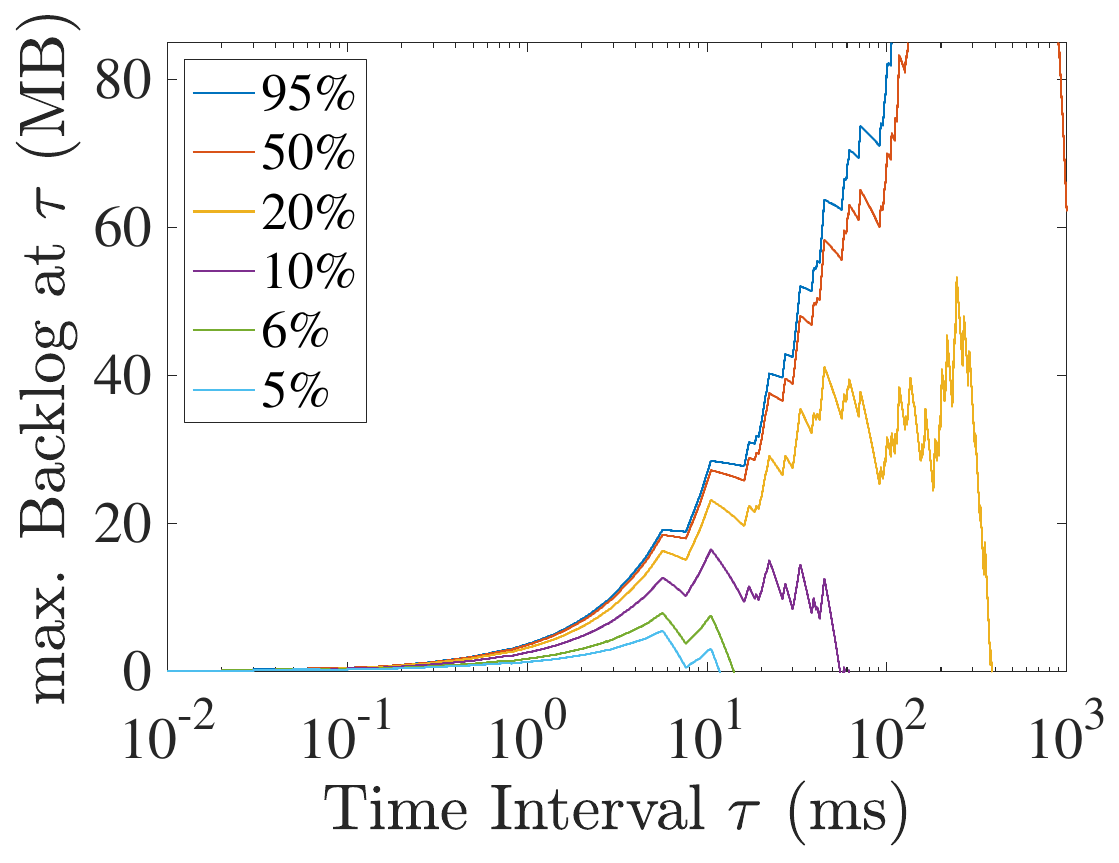}\label{fig:linear-allworkers-maxBacklog1-interval}}

\caption{Linear-Allreduce: Burstiness metrics for traffic \underline{from all workers} to the server.} 
\label{fig:linear-allworkers-burstiness}
\end{figure}
%%%------------------------------------------------------------------
Figs.~\ref{fig:linear-oneworker-burstiness} and~\ref{fig:linear-allworkers-burstiness} evaluate the burstiness metrics peak-to-mean 
ratio, maximum backlog, and interval-specific maximum backlog. Fig.~\ref{fig:linear-oneworker-burstiness} shows the metrics for the traffic from one worker (Worker3), and Fig.~\ref{fig:linear-allworkers-burstiness} those for the aggregate traffic from all workers. The peak-to-mean ratio is depicted for  time intervals up to 5\,ms, which 
covers the range most relevant to the creation of microbursts. 
For one worker, the ratio exceeds 100:1 for time intervals in the sub-millisecond range and remains around 60:1 at 5\,ms.
Observe that the peak-to-mean ratio for the aggregate traffic is noticeably smaller than that of the traffic sent by a single worker.

The maximum backlog function for one worker in Fig.~\ref{fig:linear-oneworker-maxBacklog} shows values of around 10\,MB, even at a utilization of 
only~5\%. 
For the aggregate traffic in  Fig.~\ref{fig:linear-allworkers-maxBacklog}, the maximum backlog  declines rapidly when reducing the  utilization. What is striking is that for a  utilization of 5\%, the maximum backlog for the traffic of a single worker is less than 
that for the aggregate traffic (Recall that the available rates are 
proportional to the average rate of the traffic.)

%%%------------------------------------------------------------------
\begin{figure}[t!]
\centering 
\includegraphics[width=\columnwidth]{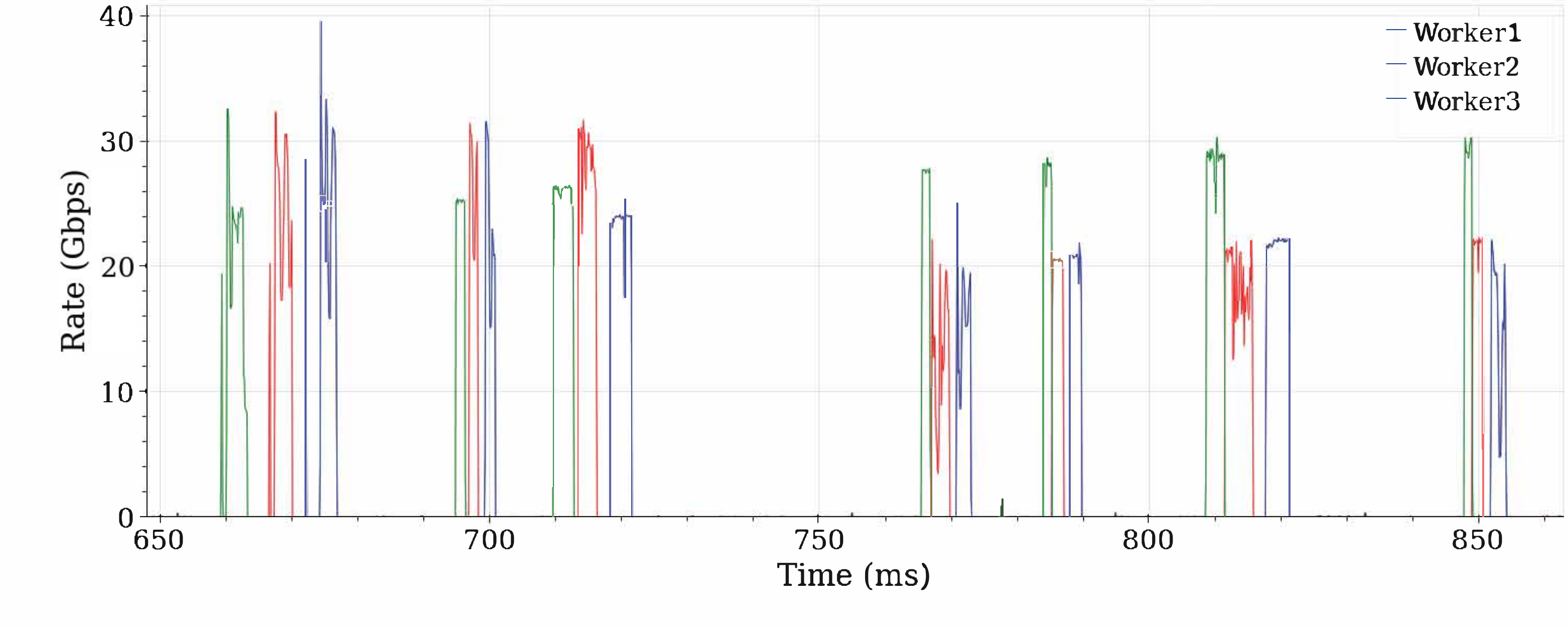}
%\subfloat[Traffic from workers to  server (200\,ms).]{\includegraphics[width=\columnwidth]{figs-color/9a.png}\label{fig:linear-allWorkers-300ms}}

%\centering 
%\subfloat[Bidirectional traffic between workers and server (100\,ms).]{\includegraphics[width=\columnwidth]{figs/linear-all-100ms-JL.pdf}\label{fig:linear-all-100ms}}
%%\subfloat[Bidirectional traffic between workers and server (100\,ms).]{\includegraphics[width=\columnwidth]{figs-color/9b.png}\label{fig:linear-all-100ms}}

\caption{Linear-Allreduce: Traffic from workers to  server in an interval of 200\,ms.} 
\label{fig:linear-allWorkers-300ms}

\end{figure}
%%------------------------------------------------------------------

The fact that the traffic from all workers is less bursty 
than the traffic from one worker may appear counterintuitive, but a 
microscopic look at the transmissions provides an explanation. Fig.~\ref{fig:linear-allWorkers-300ms} presents the transmission rates of the workers in a time interval with length of 200\,ms, 
where rates are averaged over 100\,$\mu$s intervals. 
 The graph contains the transmissions of gradients from seven layers.  
Observe that no two workers transmit to the server at the same time. Moreover, the order in which workers transmit data to the server is the same  for all transmissions. This results from the implementation of MPI primitives. With Horovod/MPI, the transmissions by workers to the  server are realized by {\sc MPI\_Reduce}, which  has workers 
transmit to the server in a fixed sequence.

%In Fig.~\ref{fig:linear-all-100ms}, we show the bidirectional traffic 
%between workers and the server for an interval of 100\,ms. 
%Traffic rates for traffic from the workers to the server are shown as thick lines (using the same grayscale colors as before).  Traffic rates for traffic from the server to the workers use thin lines. As previously observed, when workers transmit gradients to the server, they do so in a fixed sequence. 
%After the server has received gradients it performs a reduction (computes the average of gradients) and sends the results to all workers. 
% In Horovod/MPI, this is done 
%with an {\sc MPI\_Bcast} operation. As seen in Fig.~\ref{fig:linear-all-100ms}, instead of serving one worker at a time,  
%the transmissions to the workers occur concurrently. 
%As a result, the transmissions to workers are completed at about the same time. 
%Since transmissions by workers never overlap, we conclude that 
%distributed training with a  server using  Horovod/MPI does not create congestion at  the port that connects the switch to the server. This holds independent of the number of workers. 

\section{Serverless  Training}
\label{sec:RingAllreduce}
In this section, we present traffic measurements of distributed training of ResNet-50  without a coordinator. Here, the reduction of gradients is 
performed by Ring-Allreduce. Measurements are taken on 
the testbed in Fig.~\ref{fig:topo-exp}. Each host executes a worker 
for the training, denoted as Worker$i$ with $i=1,2,3,4$, with a logical 
ring 
\[ 
{\rm Worker1}\rightarrow {\rm Worker2} \rightarrow {\rm Worker3} \rightarrow {\rm Worker4} \rightarrow {\rm Worker1}\,. 
\]
Almost all traffic generated during the experiment is sent along the logical ring, and we ignore all other traffic. 
The experiment measures the traffic created by 20 rounds of 
training. 
Over the  duration of the experiment, the average rate of captured 
traffic sent by each worker  in the logical ring is around 410~Mbps. %407.3--416.3~Mbps. 

\medskip
\noindent
{\bf Remark:} In this experiment, the capture tool  sometimes fails to capture packets when  both incoming and outgoing traffic simultaneously 
have a high data rate.  The fraction of missing packets in these time periods is around 2--5\%, but has not been precisely quantified.

%%------------------------------------------------------------------
\begin{figure}[t!]
\centering 
\includegraphics[width=\columnwidth]{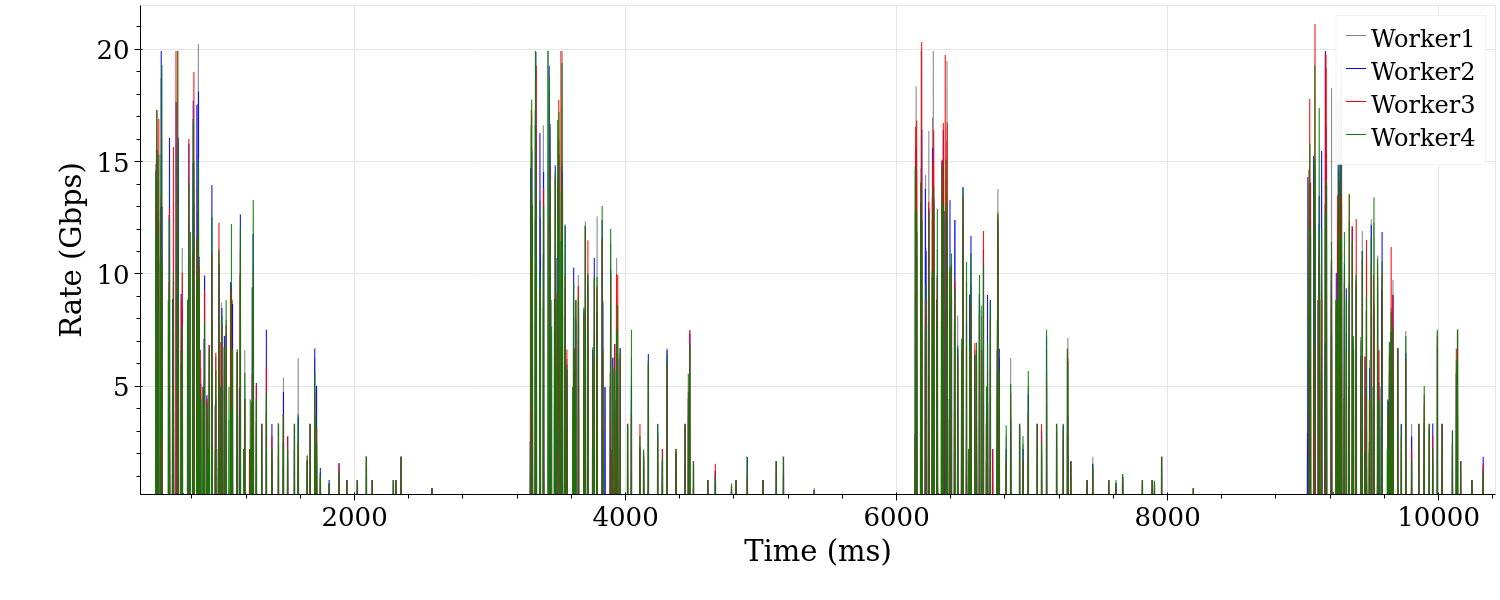}

\caption{Ring-Allreduce: Traffic between workers in logical ring (10\,s).} 
\label{fig:ring-allWorkers-10s}
\end{figure}
%---

\subsection{Traffic of Ring-Allreduce}

Fig.~\ref{fig:ring-allWorkers-10s} shows the transmission rates from the 
workers along the logical ring 
over a time period of 10\,s, where  rates are averaged over 1\,ms. The  depicted data are transmissions in the logical ring for four rounds of training. The traffic from each worker is shown in a different  color. We summarize our observations as follows:

\begin{figure}[t!]

\centering 
\includegraphics[width=\columnwidth]{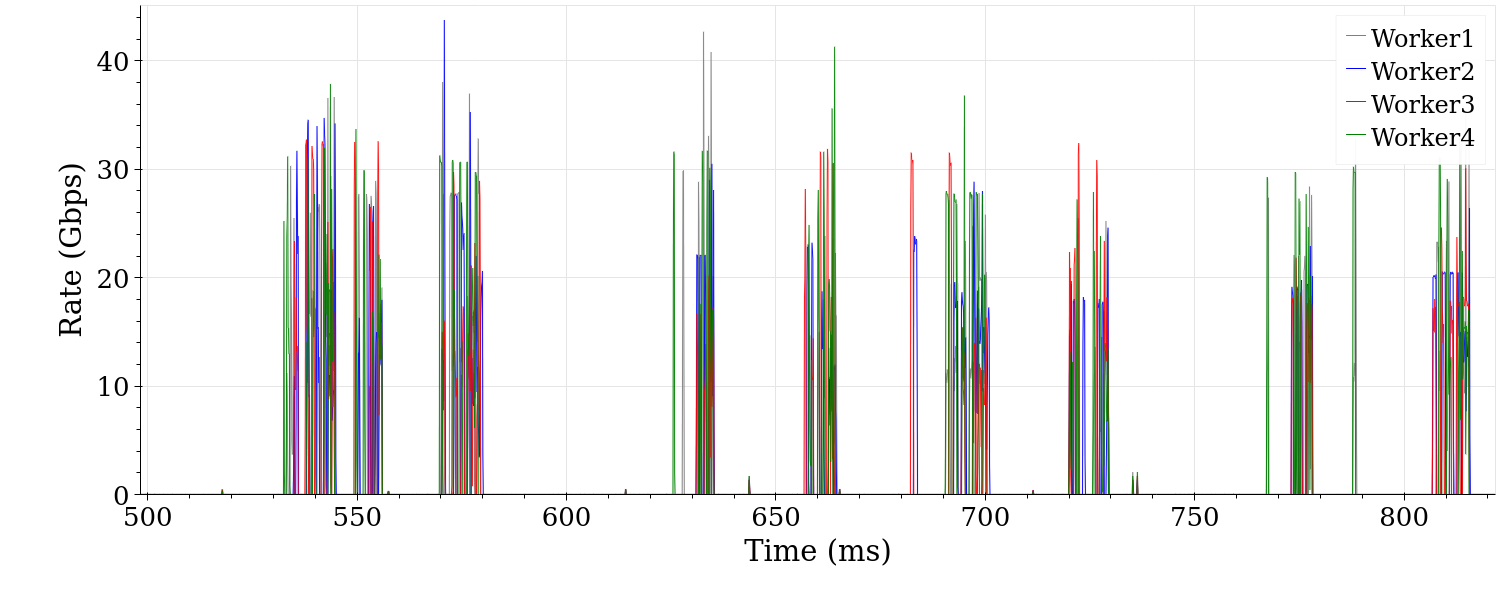}

\caption{Ring-Allreduce: Traffic between workers in logical ring (300\,ms).} 
\label{fig:ring-allWorkers-300ms}
%\end{figure}
%\begin{figure}[t!]
\centering 
\includegraphics[width=\columnwidth]{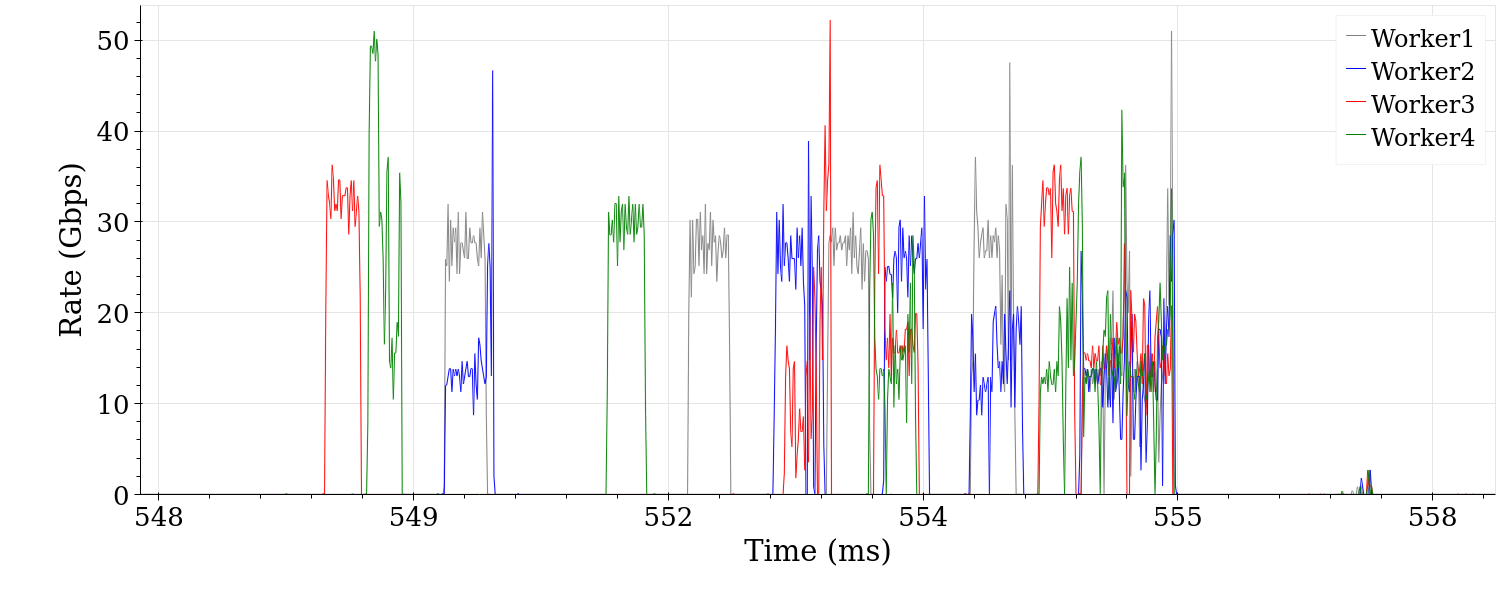}

\caption{Ring-Allreduce: Traffic between workers in  logical ring (10\,ms).} 
\label{fig:ring-allWorkers-10ms}

\end{figure}
%%------------------------------------------------------------------
\begin{itemize}
\item From the transmissions, we can infer that the time to complete a round is similar to the server-based DNN training in Sec.~\ref{sec:LinearAllreduce}. 
\item Over the depicted 10\,s, the pattern of transmissions looks 
similar to the Linear-Allreduce experiment in Fig.~\ref{fig:linear-allWorkers-10s}. Time periods of burst transmissions 
alternate with time periods with no or few transmissions. The former captures the transmission of gradients following the Ring Allreduce algorithm. The latter reflects  computations  on the forward path that do not require 
an exchange of data. 

\item 
In each time period with burst transmissions, the initial higher transmission rates reflect that the size of the transmitted gradients is larger for layers further down in the DNN. 

\item The  transmission rates, if averaged over a millisecond, rarely exceed 20\,Gbps for any of the workers, and appear to be overall lower 
than the rates seen in Fig.~\ref{fig:linear-allWorkers-10s}. This 
is due to the Ring Allreduce operation, which does not transmit the 
gradient vector in a single batch, but in smaller chunks (see Appendix~\ref{sec:primer-ringallreduce}). 
\end{itemize}

Fig.~\ref{fig:ring-allWorkers-300ms} shows the transmission rates in a time interval of 300\,ms, where  rates are averaged over 100\,$\mu$s. 
We can clearly distinguish time periods with high-rate transmissions,  
that correspond to gradient transmissions of ResNet-50 layers. 
However, the figure does not give insight into details of the 
Ring Allreduce operation. 
In Fig.~\ref{fig:ring-allWorkers-10ms}, we zoom in to a subinterval 
of 10\,ms, where each data point is an average over 10\,$\mu$s. 
The figure captures the transmissions of a single Ring Allreduce for one layer of the ResNet-50 DNN. 
At this level of detail, we are able to distinguish the transmissions from different workers:
\begin{itemize}

\item Overall, there are six bursts of transmissions for each worker. 
We can relate the first 
three bursts to the data exchange during the Reduce-Scatter phase of 
Ring Allreduce
(see Appendix~\ref{sec:primer-ringallreduce}, Fig.~\ref{fig:ReduceScatter})
and the second group of three bursts results from the 
transmissions during the Allgather phase (Appendix~\ref{sec:primer-ringallreduce}, Fig.~\ref{fig:Allgather}).

\item The first set of transmissions, around $t=549$\,ms, shows 
the first round of the Reduce-Scatter. 
We recognize four burst transmissions, one for each worker.
The second round of the Reduce-Scatter starts shortly before $t=552$\,ms. Note that the order of transmissions by the workers is different than in the first round of the Reduce-Scatter. Since there is no central coordination, the order of transmissions by the workers depends only on the 
local computation times and the times when the previous batch of data was last received from the counterclockwise neighbor in the logical ring. 
For the last round of the Reduce-Scatter and all rounds of the Allgather, the transmissions by the workers overlap in time while maintaining the 
dependencies of transmissions.

\item Note that the transmission rates in Fig.~\ref{fig:ring-allWorkers-10ms} are higher than the rates for the same time interval 
in  Fig.~\ref{fig:ring-allWorkers-10s}. This is due to the different averaging intervals used in the figures. 

\end{itemize}

\subsection{Burstiness Metrics}
\label{subsec:burstmetric-ring}

We now evaluate the burstiness metrics from Sec.~\ref{subsec:burst-metrics-def} for the traffic of 
one worker (Worker3) and for the aggregate traffic from all workers. 
We only consider traffic that is sent along the logical ring. The metrics are computed using the captured traffic from the entire experiment. 
 
\smallskip
\noindent 
{\bf Note:} In Sec.~\ref{sec:LinearAllreduce}, when we considered the aggregate traffic from all workers to the server, this traffic arrives at the switch at different ports, but leaves the switch at the same egress port. Differently, in Ring-Allreduce, gradients are sent along a logical ring,  and the aggregate traffic therefore departs the switch at different egress ports. Hence, even if all workers simultaneously 
transmit large bursts to their clockwise neighbor in the logical ring, 
the load at each egress port does not exceed its line rate. As a result, 
the backlog that accumulates at each output port is negligible.  
The graphs of the maximum backlog function 
for the aggregate traffic reflect a scenario where the  aggregate traffic  departs a switch on the same egress port.

%%------------------------------------------------------------------
\begin{figure}[t!]
\centering 
\subfloat[Peak-to-mean ratio.]{\includegraphics[width=0.32\textwidth]{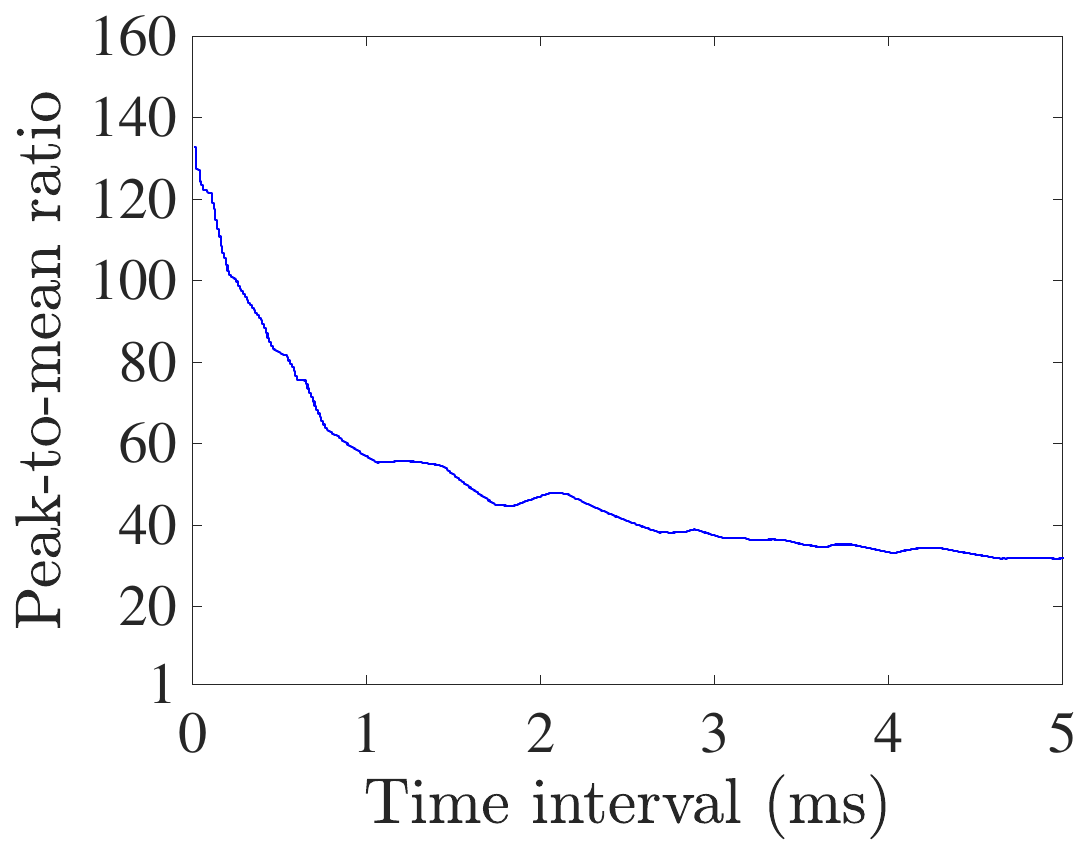}\label{fig:ring-oneworker-peak-to-mean}}
%
%\hspace{1cm}
%
\subfloat[Maximum backlog.]{\includegraphics[width=0.32\textwidth]{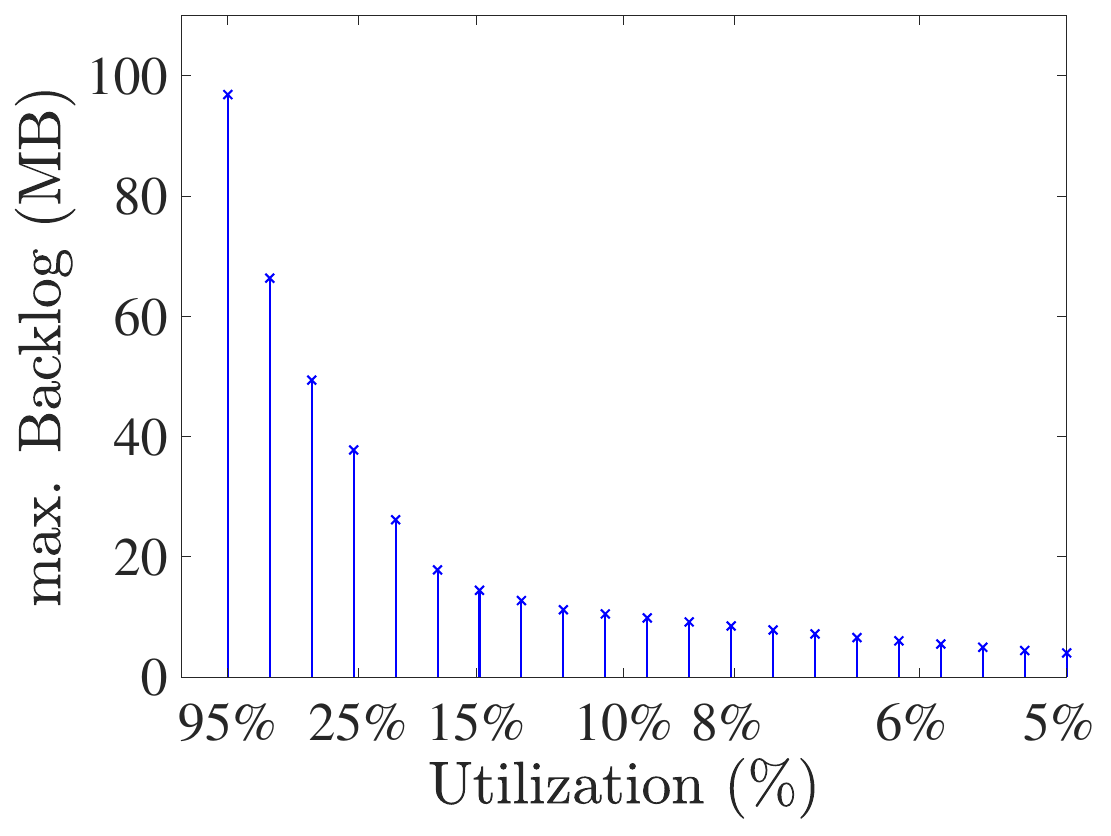}\label{fig:ring-oneworker-maxBacklog}}
\subfloat[Interval max. backlog.]{\includegraphics[width=0.32\textwidth]{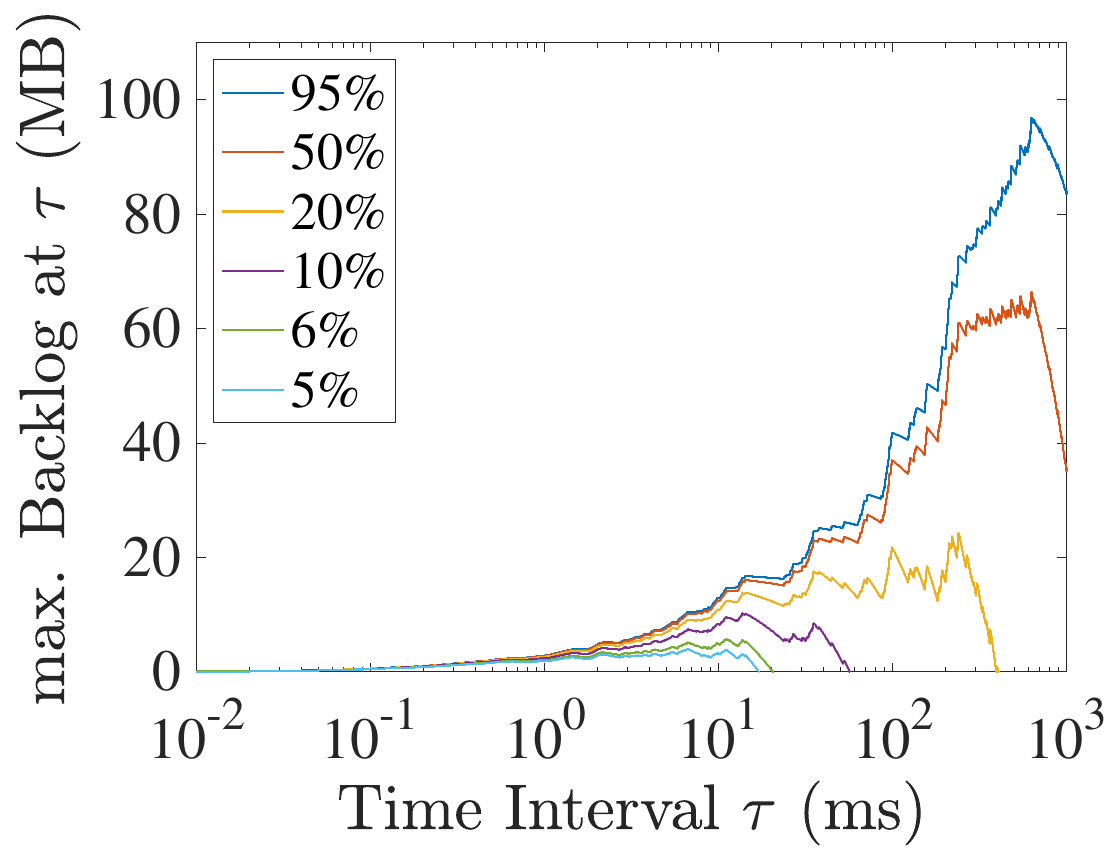}\label{fig:ring-oneworker-maxBacklog1-interval}}

\caption{Ring-Allreduce: Burstiness metrics for traffic \underline{from one worker} (Worker3).} 
\label{fig:ring-oneworker-burstiness}
%\end{figure}
%%
%\begin{figure}[t!]
\centering 
\subfloat[Peak-to-mean ratio.]{\includegraphics[width=0.32\textwidth]{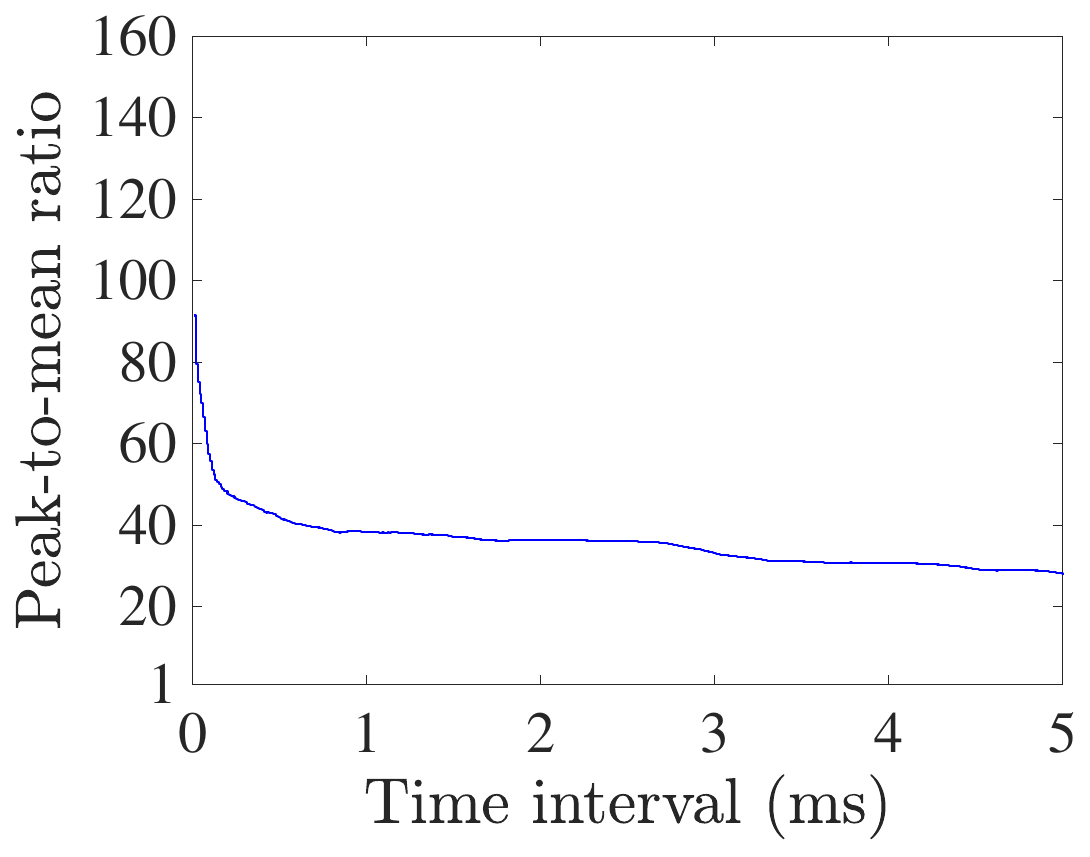}\label{fig:ring-allworkers-peak-to-mean}}
%
%\hspace{1cm}
%
\subfloat[Maximum backlog.]{\includegraphics[width=0.32\textwidth]{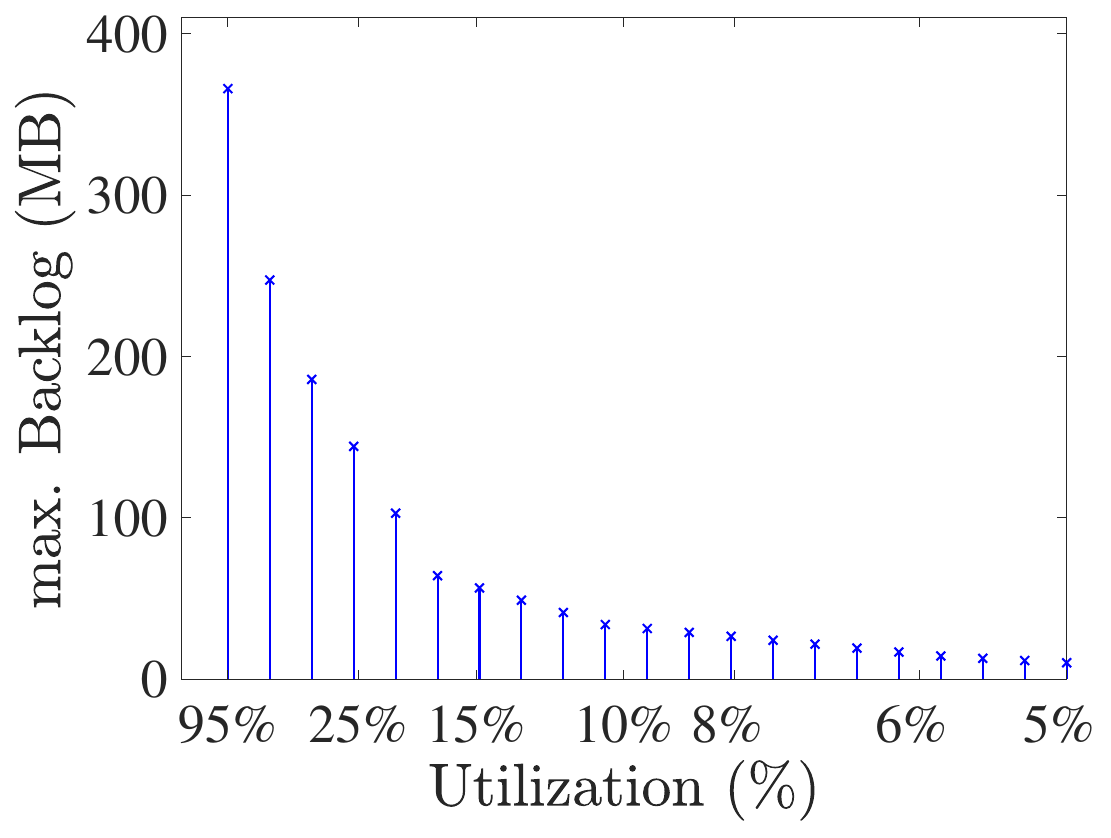}\label{fig:ring-allworkers-maxBacklog-28}}
\subfloat[Interval max. backlog.]{\includegraphics[width=0.32\textwidth]{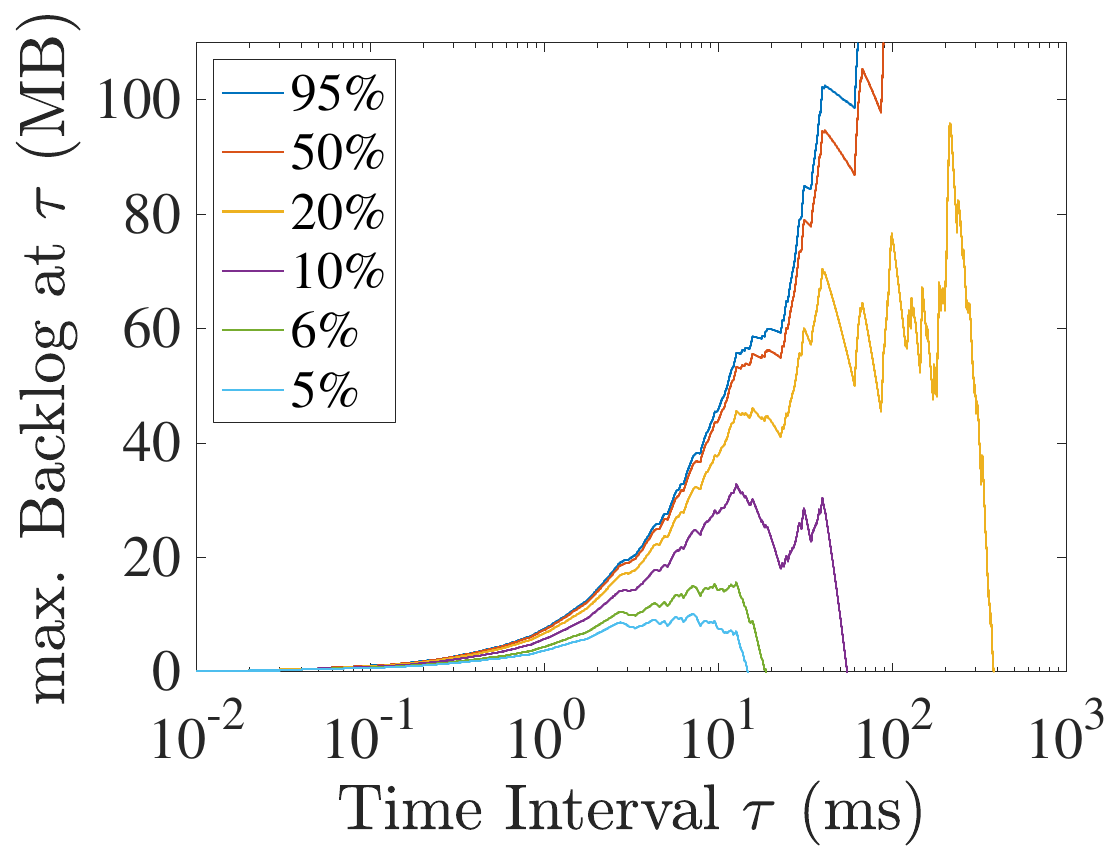}\label{fig:ring-allworkers-maxBacklog1-interval}}

\caption{Ring-Allreduce: Burstiness metrics for traffic \underline{from all workers}.} 
\label{fig:ring-allworkers-burstiness}
\end{figure}
%%%------------------------------------------------------------------

\bigskip
Figs.~\ref{fig:ring-oneworker-burstiness}  and~\ref{fig:ring-allworkers-burstiness} compare the burstiness metrics 
of the traffic from one worker to those of the aggregate traffic from all workers. Different from Linear-Allreduce, the peak-to-mean ratios for intervals >1~ms for one worker and all workers are similar. 
This also holds for the maximum backlog functions in Figs.~\ref{fig:ring-oneworker-maxBacklog} and~\ref{fig:ring-allworkers-maxBacklog-28}. 
For both a single worker and all workers, there is a linear decay of the maximum backlog, which we interpret as being associated with microbursts. 
A comparison of Figs.~\ref{fig:ring-oneworker-maxBacklog1-interval} and~\ref{fig:ring-allworkers-maxBacklog1-interval} shows that the 
maximum backlog at a low utilization is larger for the aggregate traffic than for a single worker. Recall that the opposite was the case in the Linear-Allreduce experiment. 
The peak-to-mean ratio for Worker3 is initially above 100 and decreases to below 40 for time intervals longer than 3\,ms. If we compare this to the traffic of the same worker in  the Linear-Allreduce experiment (Fig.~\ref{fig:linear-oneworker-burstiness}), we see that the peak-to-mean ratio with Ring-Allreduce is lower. This is a consequence of the Ring Allreduce algorithm 
(see Appendix~\ref{sec:primer-ringallreduce}, Figs.~\ref{fig:ReduceScatter} and~\ref{fig:Allgather}) as burst sizes sent by one worker in a Ring Allreduce with a total of four workers are $1/4$-th of the burst sizes sent to the server in Linear-Allreduce. The maximum backlog function for Worker3 in Fig.~\ref{fig:ring-oneworker-maxBacklog} shows that even at a link utilization of 5\%, the backlog can be several MB. A comparison with the backlog function for Linear-Allreduce (in Fig.~\ref{fig:linear-oneworker-maxBacklog}) yields that the maximum backlog at a low utilization is smaller with Ring-Allreduce. The reason, again,  is the smaller burst sizes found in Ring-Allreduce. 

Overall, with Ring-Allreduce, the  aggregate traffic from all workers has similar 
burstiness characteristics as the traffic from a single worker. This is notably different from the Linear-Allreduce experiment, where due to 
the coordination of transmissions by workers, the 
aggregate traffic from all workers is less 
bursty than the traffic from a single worker. Since such a coordination is 
absent in Ring-Allreduce, 
the burstiness metrics do not decline when 
aggregating the traffic from multiple workers. 
As noted earlier, in our measurement experiment, the aggregate traffic in the Ring-Allreduce experiment is split across different egress ports. On the other hand, if the same traffic is transmitted in a DCN with a leaf-spine topology, traffic from multiple workers may reach the same egress  port of a switch.  Here, the risk of microbursts largely depends on the placement of worker nodes in the DCN topology.  

\section{Burstiness Potential of Distributed ML  Traffic}
\label{sec:unorchestrated-ML}

Fan-in has been identified as a significant risk for the creation of microbursts~\cite{microburst2-1}. 
It generally refers to the multiplexing of concurrent flows belonging to the 
same application that simultaneously transmit  
data to the same target address.   Concurrent transmissions from multiple 
sources to the same target may create a bottleneck at switches located 
close to the target,
which may experience large burst arrivals destined to the same output 
port. 
Concurrent traffic flows to the same target are common in many DCN  applications, such as 
MapReduce~\cite{mapreduce}, distributed storage~\cite{distr-storage}, 
Memcached~\cite{memcached}, and distributed machine 
learning~\cite{tensorflow,pytorch}. However, as seen in the traffic traces of distributed training of the ResNet-50 model in the previous sections, this does not 
necessarily result in a large backlog at switches. 

\begin{itemize}
\item In the Linear-Allreduce experiment (Sec.~\ref{sec:LinearAllreduce}), 
in each layer of the backpropagation algorithm, all workers send sometimes large bursts of data to the same destination. These bursts appear to be concurrent,  however, inspections of the traffic at a small time scale and a study of the underlying software reveals that, at any time, only one worker transmits traffic to the server. 

\item In Sec.~\ref{sec:RingAllreduce}, the design of the Ring Allreduce algorithm  ensures that fan-in does not occur, at least not in a single switch topology. 
\end{itemize}
We therefore step away from the consideration of microbursts created by a single application and consider scenarios where traffic from multiple  concurrently running applications have traffic that is destined to the same egress port of a network switch. We refer to such scenarios as cross-application fan-in.  

We next use the traffic traces from the previous sections to 
get a sense how such a cross-application fan-in may look like 
for DNN traffic. For this we, explore the 
worst-case burstiness of the  captured traffic from the experiments in 
Secs.~\ref{sec:LinearAllreduce} and~\ref{sec:RingAllreduce}. 
Traffic from multiple traffic flows has maximal burstiness if it 
aligns in the worst possible fashion. Within a single application, 
both Linear-Allreduce and Ring-Allreduce prevent that traffic aligns 
in a worst-case fashion at one egress port. The considered worst-case considerations therefore apply to scenarios where each traffic flow arrives from a different application. 
%%------------------------------------------------------------------
\begin{figure}[t]
\centering 
\subfloat[Arrivals are not aligned.]{\includegraphics[width=0.4\textwidth]{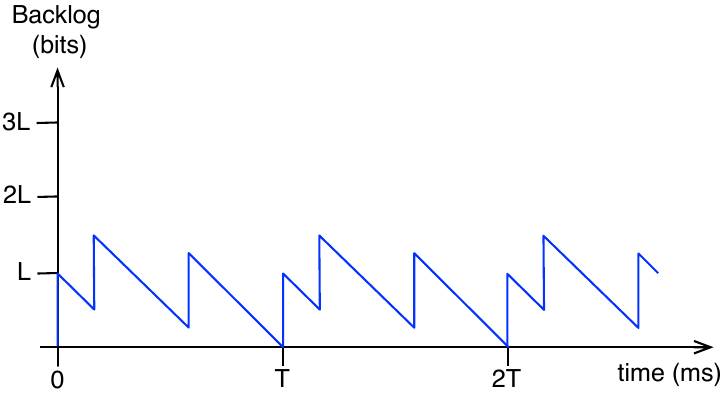}
\label{chp5-fig:fig5-arrival-scenario2}}
\subfloat[Arrivals are aligned in a worst-case fashion.]
{\includegraphics[width=0.4\textwidth]{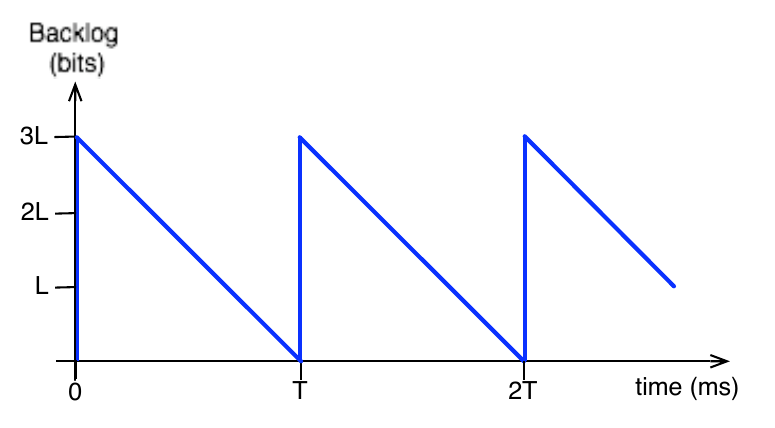}
\label{chp5-fig:fig5-arrival-scenario1}}

\caption{Backlog for 3 flows
with periodic arrivals.} 
\label{chp5-fig:fig5-arrival-scenarios}
\end{figure}
%------------------------------------------------------------------

\smallskip\noindent
{\bf Example:} 
Consider a set of three periodic flows that each issue a burst of size $L$\,bits every $T$\,ms to a switch port with line rate $C$\,kbps. 
If the packet arrivals of the flows are not aligned, the backlog at the switch port may evolve as shown in Fig.~\ref{chp5-fig:fig5-arrival-scenario2}. Backlog decreases at a constant rate  $C$, and each  packet arrival increases the backlog by $L$.  Note that the backlog never exceeds  $2L$. Fig.~\ref{chp5-fig:fig5-arrival-scenario1} shows the backlog when the arrivals from the three flows are aligned. This scenario creates the worst-case burstiness and results in the worst-case backlog of $3L$.

\smallskip
To generalize the example, consider the aggregation of a set of traffic flows with arrival functions $A_j$ and burstiness curves $\E_j$ ($j=1,2, \ldots, N$). The arrival function of the aggregate, 
denoted by $A$, is given by 
\[
A (t) = \sum_{j=1}^{N} A_j(t) \, . 
\]
The burstiness curve of the aggregate traffic, denoted by 
$\E_A$,  is computed as   
\[
\E_A (t) = \Bigl(\sum_j A_j \Bigr) \deconv \Bigl(\sum_j A_j \Bigr) (t) \, . 
\]
An upper bound for $\E_A$ is the sum of the burstiness curves 
of 
all flows, which is given by 
\[
 \sum_j \E_{A_j} (t) = \sum_j \Bigl( A_j  \deconv  A_j \Bigr) (t) \, . 
\]
The sum $\sum_j \E_{A_j}$ therefore provides the worst-case burstiness of the aggregate traffic. If $\E_A \approx  \sum_j \E_{A_j}$  we can conclude that traffic is close to a worst-case alignment. In the example above, where each flow $j$ has the same  burstiness curve of  
$\E_{A_j} (t) =  L \lceil t / T \rceil$,  the traffic scenario in 
Fig.~\ref{chp5-fig:fig5-arrival-scenario1} satisfies $\E_A =  \sum_j 
\E_{A_j}$, whereas 
Fig.~\ref{chp5-fig:fig5-arrival-scenario2} yields $\E_A \ll  \sum_j 
\E_{A_j}$. The difference $\E_A -  \sum_j 
\E_{A_j}$ indicates the degree to which the burstiness of a set of flows may increase by 
a different alignment (time shift) of the traffic. We refer to this difference as the 
{\it burstiness potential}. We next explore the burstiness potential of 
the captured  traces from the measurement experiments in Secs.~\ref{sec:LinearAllreduce} and~\ref{sec:RingAllreduce}.

%\subsubsection{Burstiness Potential of Linear Allreduce Traffic}

Let us consider the aggregate traffic of the workers to the servers from the Linear-Allreduce experiment in Sec.~\ref{sec:LinearAllreduce}.  
Figs.~\ref{fig:linear-Compare-peak-to-mean-ratio} and~\ref{fig:linear-Compare-maxBacklog} compare the peak-to-mean ratio and the maximum backlog 
metrics  of the aggregate traffic of the workers ($\E_A$,  labeled as `actual') with the metrics obtained by adding the burstiness curves of 
the workers ($\sum_j \E_{A_j}$, labeled as `worst-case').  
The large difference of the metrics of the actual versus the worst-case 
traffic is due to the orchestration of transmissions by the workers, where only one worker transmits to the server at a time. The worst-case scenario shows 
the traffic that may result without such an orchestration.

Figs.~\ref{fig:ring-Compare-peak-to-mean-ratio} and~\ref{fig:ring-Compare-maxBacklog} evaluate the burstiness potential for the traces from the Ring-Allreduce experiment in Sec.~\ref{sec:RingAllreduce}. A comparison of the corresponding 
metrics for the Linear-Allreduce trace gives insight into  differences of the  burstiness characteristics.  For instance, the gap between the burstiness metrics in Ring Allreduce are less than seen for Linear-Allreduce. 
This indicates that the aggregate Ring Allreduce traffic from all workers is not very different from its worst-possible alignment. 

%%------------------------------------------------------------------
\begin{figure}[t!]
\centering
\subfloat[Linear-Allreduce: Peak-to-mean ratio.]
{\includegraphics[width=0.42\textwidth]{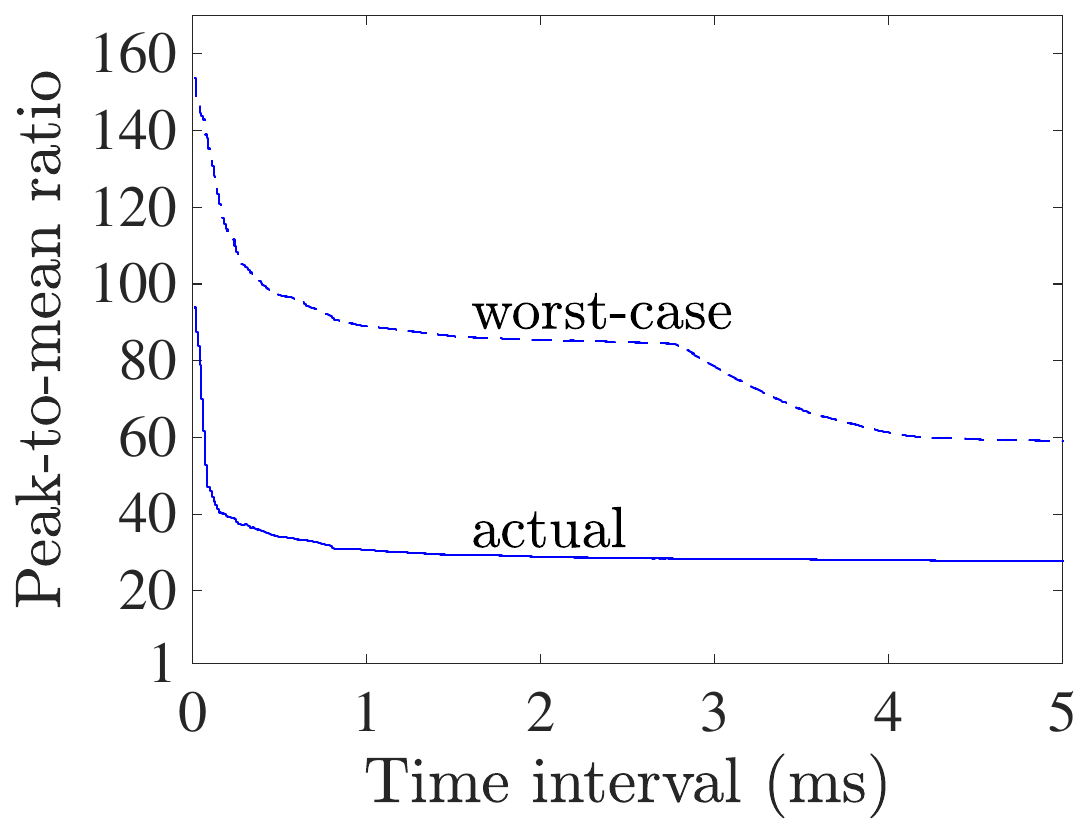}
\label{fig:linear-Compare-peak-to-mean-ratio}}
\hspace{1cm} 
\subfloat[Linear-Allreduce: Maximum backlog.]
{\includegraphics[width=0.42\textwidth]{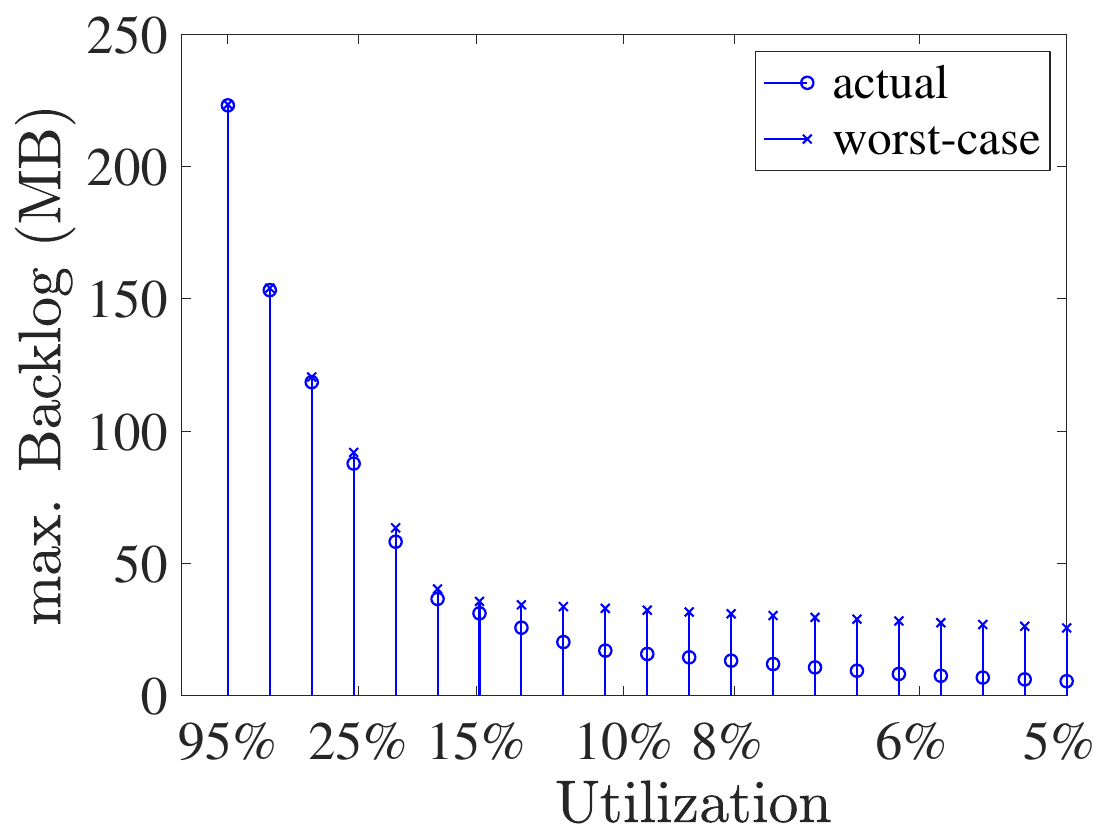}
\label{fig:linear-Compare-maxBacklog}}

\vspace{-10pt}
\centering
\subfloat[Ring-Allreduce: Peak-to-mean ratio.]
{\includegraphics[width=0.42\textwidth]{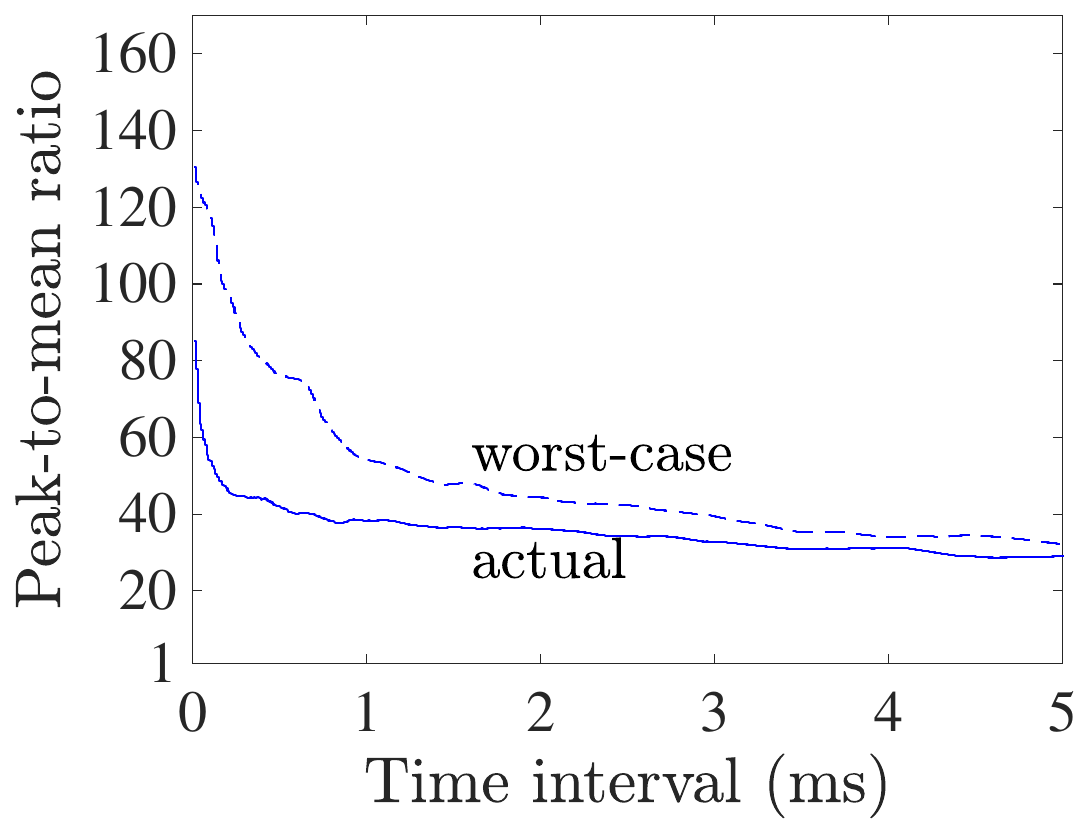}
\label{fig:ring-Compare-peak-to-mean-ratio}}
\hspace{1cm} 
\subfloat[Ring-Allreduce: Maximum backlog.]
{\includegraphics[width=0.42\textwidth]{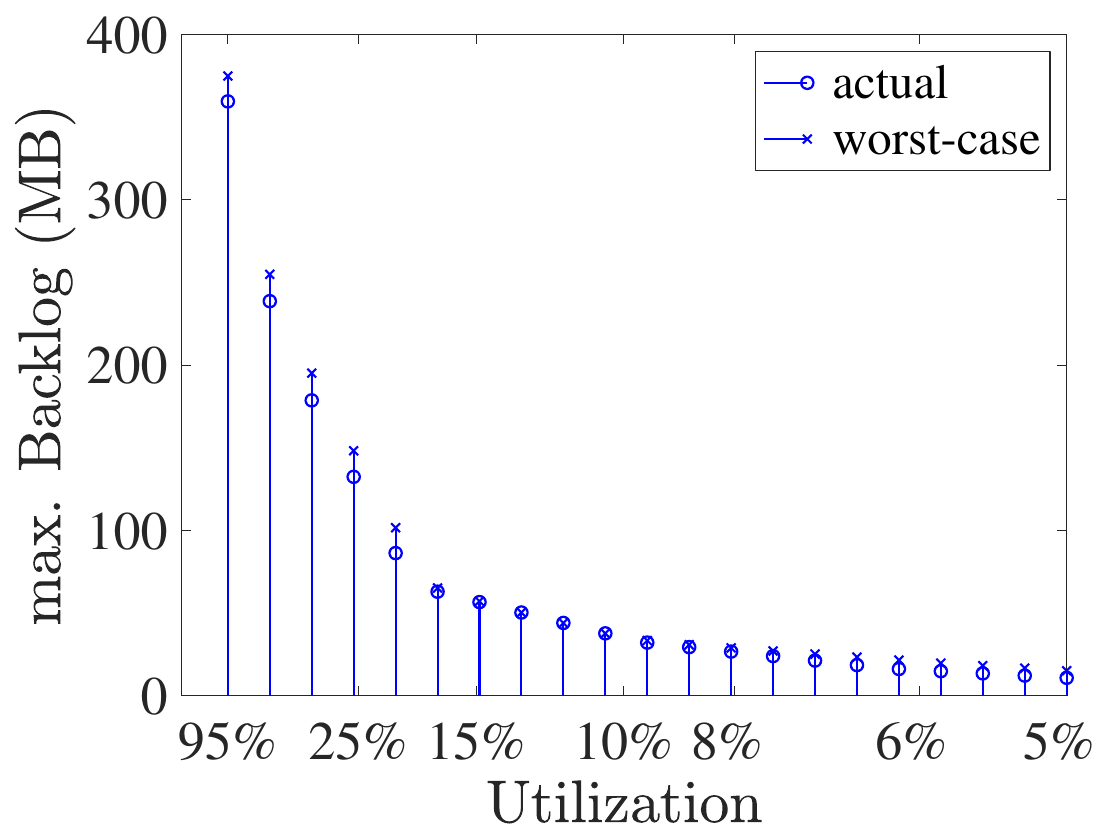}
\label{fig:ring-Compare-maxBacklog}}
\caption{Comparison of actual and worst-case burstiness of traffic.} 
\label{fig:Compare-BurstinessMetrics}

\vspace{-10pt}

\end{figure}
%%---

\section{Distributed ML Traffic and Congestion Control}
\label{subsec:simulate-model-MB}
In this section, we use a simulation to demonstrate how congestion control 
deals with simultaneous bursts from distributed training  
due to cross-application fan-in of concurrent ML applications 
in a DCN. The simulations are built using the ns-3 simulator for DCNs \cite{rdmasim-code}.

In the simulation setup, hosts are connected to 
a shared memory 32-port switch with 100\,Gbps links. The propagation delay of each link is set to 1\,$\mu$s.
There are 30 hosts that act as workers of a distributed ML 
program. Each worker independently generates traffic of a worker in the Linear-Allreduce scenario for training of a ResNet-50 model. For each layer of the model, the burst size and time gap until the next burst are slightly randomized  using a uniform distribution in the range of values observed in the measurement traces. 
All workers send their traffic to the same destination. We assume that each worker belongs to a different application and that the transmissions of the 
workers are not coordinated, i.e., multiple workers can transmit at the same time. 
The simulation does not account for processing times at the destination. Hence, the destination does not become a bottleneck in the simulation. 
As in the measurement experiments of Secs.~\ref{sec:LinearAllreduce} and~\ref{sec:RingAllreduce},  data is sent using 
RDMA with PFC and DCQCN enabled, where the parameters for PFC and DCQCN in 
the simulation are set to the same values as in the measurement experiments 
(Table~\ref{table:PFC-DCQCN-parameters} in Appendix~\ref{sec:dcqcn-pfc} lists 
the parameter values). 
 
The simulation investigates the aftermath of an event where the transmissions of 
gradients from the last layer of the ResNet-50 model arrive simultaneously at the switch. 
Recall that the gradients from the last layer are the first batch of gradients that  a worker of a Linear-Allreduce application transmits to the server on the backward pass (Fig.~\ref{fig:DNN-training}).
Such a scenario creates a worst-case congestion event at the switch for two reasons. First, the last layer has the most parameters and, therefore, the transmissions of 
gradients create the largest traffic bursts. Moreover, since there are no transmissions to the server on the forward pass, there is a long idle period 
that precedes the gradient transmissions for the last layer. 
Since DCQCN increments the permitted transmission rate, if no congestion event is reported to the sender, the transmission rate of a worker is maximal at the start of the backward pass. The observed idle times preceding the transmission on the backward pass are long enough so that the permitted transmission rates equal the line rate of 100\,Gbps.

%%------------------------------------------------------------------
\begin{figure}[t!]
\centering 
%
%\subfloat[Total rate of arriving traffic.]{\includegraphics[width=\columnwidth]{figs/unorchestrated_linear_allreduce_30workers_delay1us_total10ms_rate__avginterval_10us.png}
\subfloat[Total rate of arriving traffic.]{\includegraphics[width=0.7\textwidth]{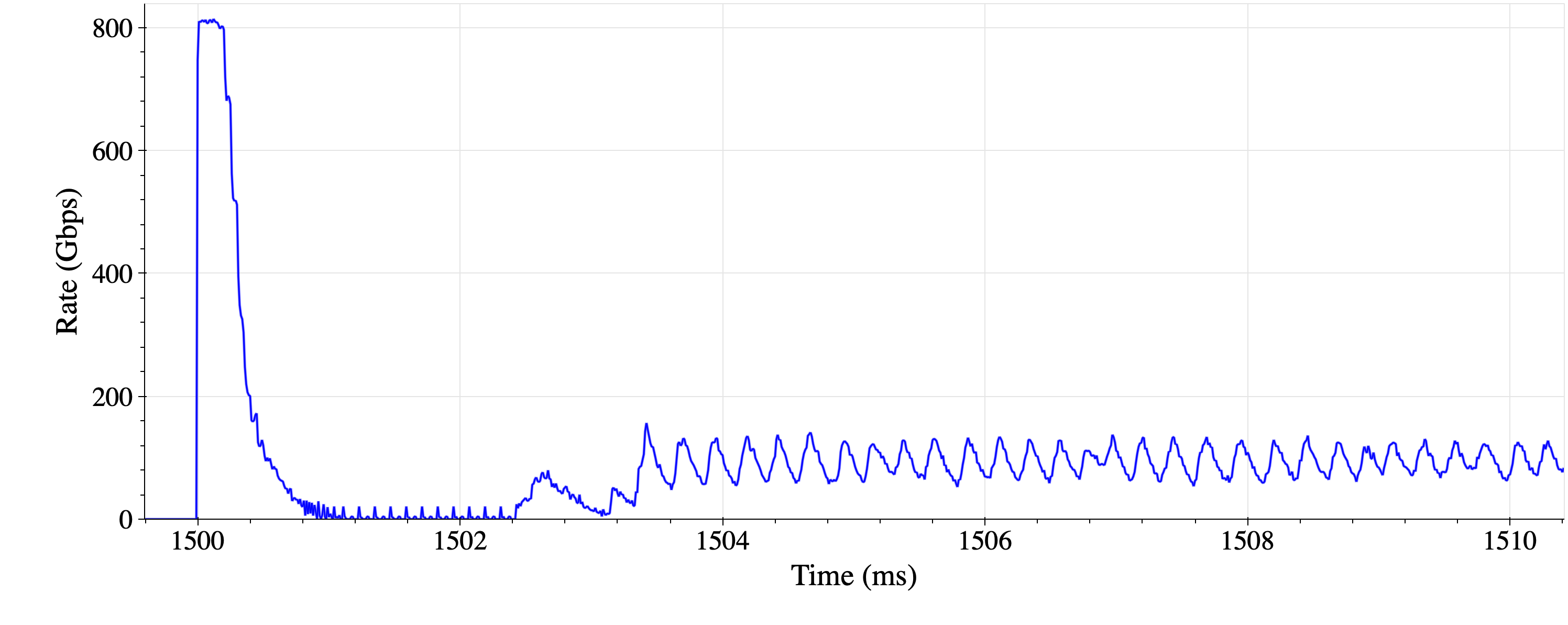}\label{fig:unorchestrated_linear_allreduce_30workers_delay1us_total10ms_rate}}
\hspace{5mm}
%\centering 
%\subfloat[Backlog.]{\includegraphics[width=\columnwidth]{../figs-bw/unorchestrated_linear_allreduce_30workers_delay1us_total10ms_backlog__avginterval_10us.png}\label{fig:unorchestrated_linear_allreduce_30workers_delay1us_total10ms_backlog}}
\subfloat[Backlog.]{\includegraphics[width=0.7\textwidth]{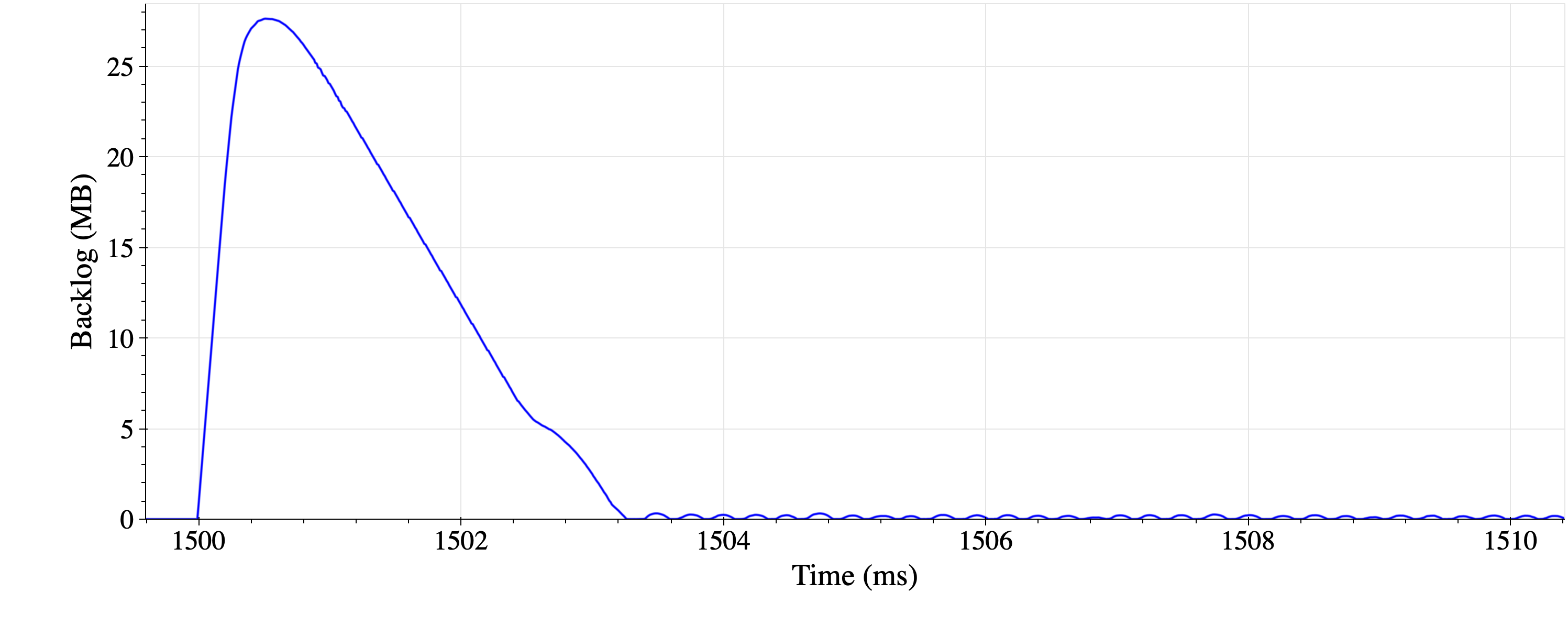}\label{fig:unorchestrated_linear_allreduce_30workers_delay1us_total10ms_backlog}}

\caption{Simulation of distributed ML traffic with simultaneous bursts (Time interval:  10\,ms).} \label{fig:unorchestrated_linear_allreduce_30workers_delay1us_total10ms}

\vspace{-6pt}
\end{figure}
%%------------------------------------------------------------------

Fig.~\ref{fig:unorchestrated_linear_allreduce_30workers_delay1us_total10ms_rate} shows the traffic rates created by all workers and the backlog at the switch port for a time period of 10\,ms, starting at  $t=1500$\,ms. 
The initial rates  reflect the created scenario where the gradients from the last layer of ResNet-50 are sent  simultaneously by all workers to the switch. Matching the experimental data, each worker transmits bursts  at a rate between $25$ and $35$\,Gbps, as observed in the measurement experiments. 
The  total arrival rate in 
Fig.~\ref{fig:unorchestrated_linear_allreduce_30workers_delay1us_total10ms_rate}
initially jumps to 800\,Gbps and then drops off.
 After the initial period, which is only a fraction of a millisecond, 
DCQCN reduces the total traffic rate  in several steps until it almost reaches  zero, after which it increases again at a small slope. During the rise of the rate, which continues past the shown interval, the arrival rates with {\it DCQCN} oscillate in an almost periodic pattern. 
This oscillation is due to PFC, which shuts off transmissions from a host if 
the backlog of an input buffer exceeds a threshold. 
The backlog plot in Fig.~\ref{fig:unorchestrated_linear_allreduce_30workers_delay1us_total10ms_backlog} shows a long initial backlogged period that grows close to the maximum total backlog permitted by~PFC (set to 9.5 kB/port/Gbps). The initial backlogged period has a duration of several milliseconds.

Fig.~\ref{fig:unorchestrated_linear_allreduce_30workers_delay1us_total800us} zooms in  on the initial backlogged period  by showing 
the time interval $[1500,1500.8]$\,ms. For the DCQCN algorithm, the total rate of arriving traffic in Fig.~\ref{fig:unorchestrated_linear_allreduce_30workers_delay1us_total800us_rate} remains at 800\,Gbps for more than 200\,$\mu$s before it decreases for the first time. 
This raises the question {\it why the workers in Fig.~\ref{fig:unorchestrated_linear_allreduce_30workers_delay1us_total800us_rate} wait 200\,$\mu$s before they reduce their transmission rate to the switch for the first time?} 
The explanation for the delayed rate reduction lies in the operation of 
DCQCN. 
When no transmissions take place, DCQCN 
periodically increases the transmission rate (using additive increase) up to the line rate.

%%------------------------------------------------------------------
\begin{figure}[t!]
\centering 
%
%\subfloat[Total rate of arriving traffic.]{\includegraphics[width=\columnwidth]{../figs-bw/unorchestrated_linear_allreduce_30workers_delay1us_total800us_rate__avginterval_1us.png}\label{fig:unorchestrated_linear_allreduce_30workers_delay1us_total800us_rate}}
\subfloat[Total rate of arriving.]{\includegraphics[width=0.7\textwidth]{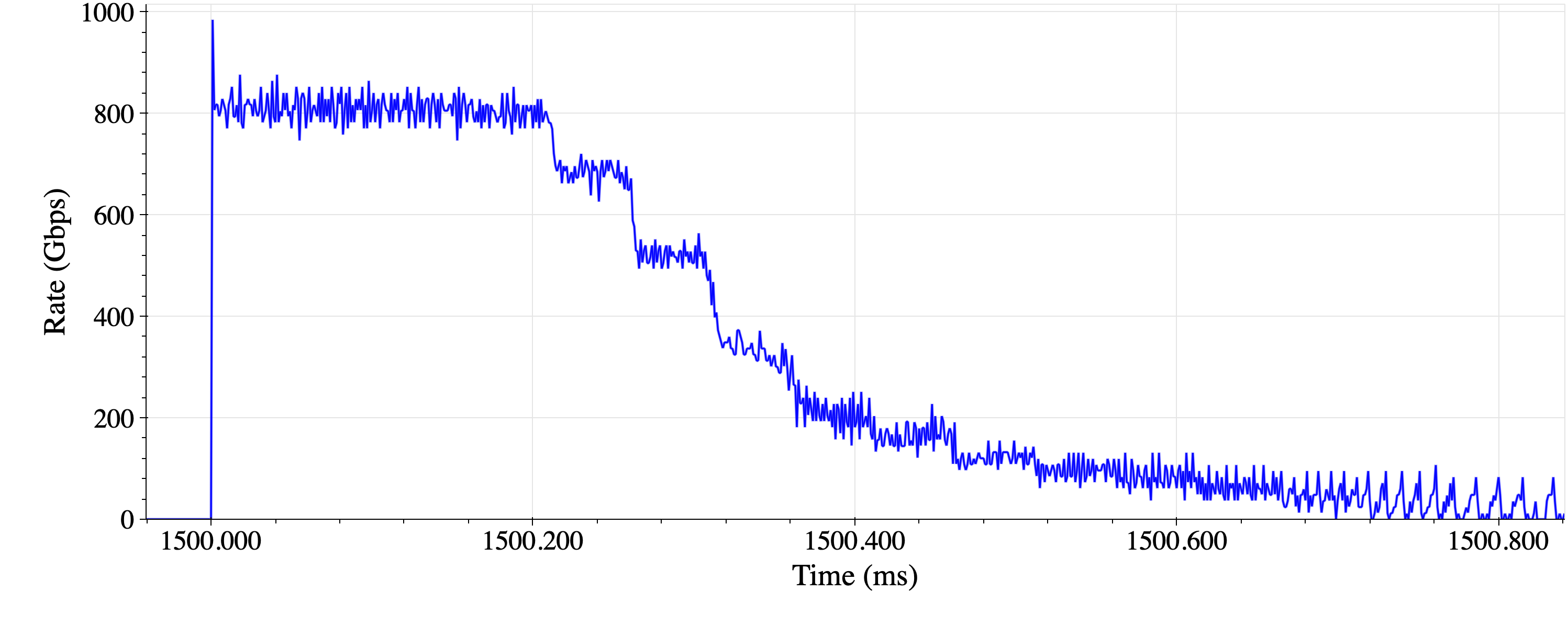}\label{fig:unorchestrated_linear_allreduce_30workers_delay1us_total800us_rate}}
\hspace{5mm}
%\centering 
%\subfloat[Backlog.]{\includegraphics[width=\columnwidth]{../figs-bw/unorchestrated_linear_allreduce_30workers_delay1us_total800us_backlog__avginterval_1us.png}\label{fig:unorchestrated_linear_allreduce_30workers_delay1us_total800us_backlog}}
\subfloat[Backlog.]{\includegraphics[width=0.7\textwidth]{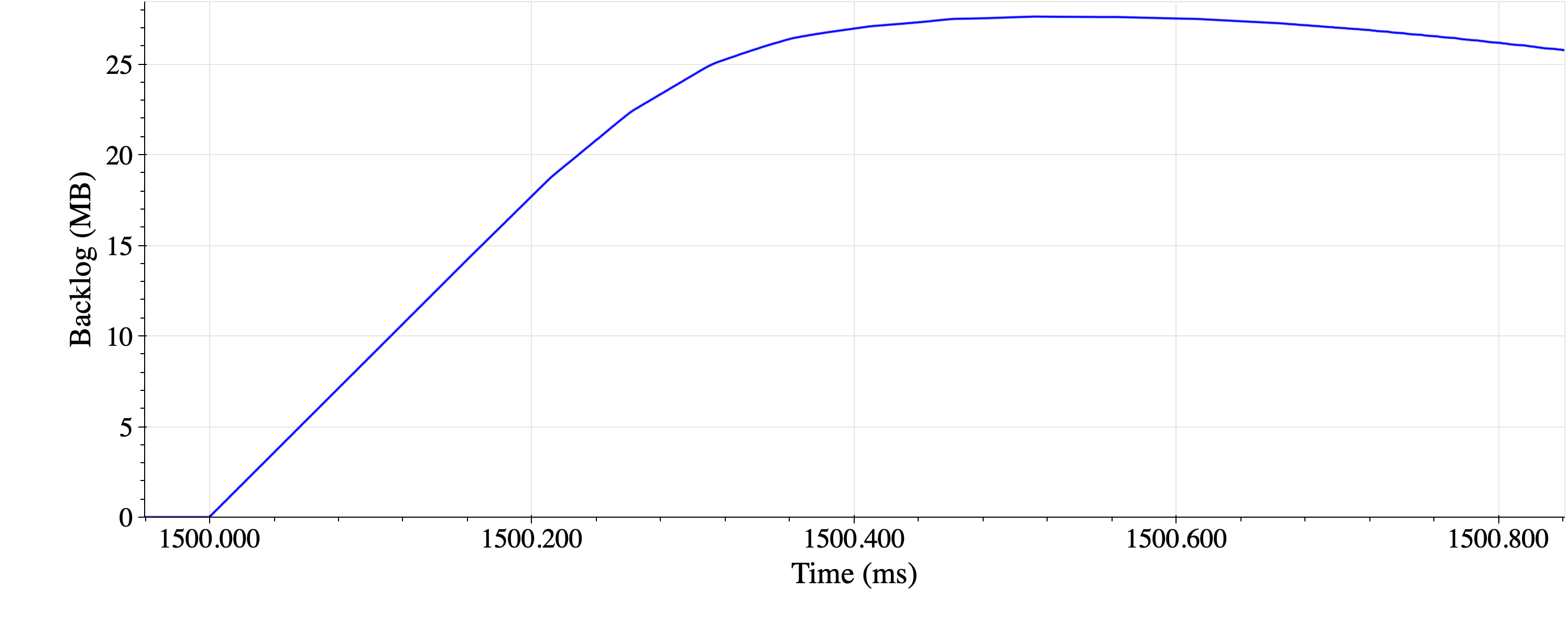}\label{fig:unorchestrated_linear_allreduce_30workers_delay1us_total800us_backlog}}

\caption{Simulation of distributed ML traffic with simultaneous bursts (Time interval:  800\,$\mu$s).} \label{fig:unorchestrated_linear_allreduce_30workers_delay1us_total800us}

\vspace{-10pt}
\end{figure}
%%------------------------------------------------------------------

Now, when the workers start to transmit at $t=1500$\,ms,
a congestion notification packet (CNP) indeed reaches each worker within a few microseconds. This results 
in a multiplicative decrease of the permitted transmission rate, which is reduced from 100\,Gbps to around 75\,Gbps. However, since the initial 
transmission rate of each  worker is  only 20--35\,Gbps, the reduction of 
the permitted transmission rate has no impact.
Only when DCQCN sets the permitted rate to below 35\,Gbps do  workers 
actually reduce their transmission rates. Since rate reductions are triggered by an arrival of a CNP, and CNPs are sent only once every 50\,$\mu$s (using the 
 value suggested in \cite{DCQCN-Sigcomm2015}), it takes several rounds of CNP transmissions before the first rate reduction is observed.

The simulation of simultaneous large bursts, as observed in DNN training, points to shortcomings of the DCQCN protocol. When congestion occurs, DCQCN reduces the maximum permitted rate 
of the sources. This reduces the actual transmission rates of the sources only if sources have enough data to transmit. As observed, if an application generates data at a rate that is well below the maximum permitted rate, a reduction of the permitted rate has no impact. In 
Fig.~\ref{fig:unorchestrated_linear_allreduce_30workers_delay1us_total800us_rate}, it takes multiple CNP generation intervals before 
any rate reduction takes place. This delay in reacting to a congestion event is inversely proportional to the rate at which applications generate data. The smaller the rate at which a source wants to transmit, the longer the delay until it reacts to a congestion event. 

The delayed response to congestion events has stark consequences. In 
Fig.~\ref{fig:unorchestrated_linear_allreduce_30workers_delay1us_total800us_backlog}, we see that at time $t=1500.2$\,ms, shortly before the first rate reduction, the backlog is around~17\,MB. Had the sources reacted to the first congestion notification after around 10\,$\mu$s, the backlog would have been only around~1\,MB.

\section{Related Work}
\label{sec:related}

{\bf Distributed ML traffic:} 
There is a wealth of recent studies on profiling DNN training \cite{Kesidis22,DDN-perf1,DDN-perf2,DDN-perf3,DDN-perf4,DDN-perf5,DDN-perf6} and reducing communication times \cite{DDN-netperf1,DDN-netperf2,DDN-netperf3,DDN-netperf4}, however, a systematic study of the communication overhead and traffic patterns 
of DNN training does not exist. 
To the best of our knowledge, there is only one prior measurement study of distributed ML traffic \cite{Cambridge1}, with traces available in~\cite{pscode}. These measurements are performed  on a testbed similar to Fig.~\ref{fig:topo-exp} with 10\,Gbps line rates and with TCP as transport protocol.  The trained neural network is a simple  Multilayer Perceptron (MLP) with only two hidden layers and a small parameter set, which completes a round of server-based training in about 5~ms. 
The distributed training simulator  ASTRA-sim~\cite{astrasim1,astrasim2} presents an alternative to measurement experiments. ASTRA-sim can generate network traffic for 
 given DNN models and network configurations, which, in turn, can be used to drive network simulations \cite{astrasim3}.

{\bf Microbursts:} The problems posed by microbursts have been studied for about a decade. 
Work in this area is concerned with measuring and understanding 
microbursts in DCNs 
\cite{microburst1,microburst2,microburst5,microburst6,microburst2-1,microburst-nsdi22, microburst-nsdi23}, as well 
as developing methods for absorbing or reducing microbursts 
\cite{microburst3,microburst4,microburst7,microburst8,microburst9,HPCC-Sigcomm19,Swift-Sigcomm20,microburst-Sigcomm17}. 
The typical workload in these studies consists of Web traffic, distributed caching, and MapReduce, but not distributed ML traffic. 
Microbursts are sometimes defined in terms of specific traffic patterns 
\cite{microburst-nsdi23} or in terms of the impact they cause, such as buffer utilization~\cite{microburst6}, level of congestion~\cite{microburst-Sigcomm17} or loss~\cite{microburst-CCR}. 
Since switch vendors are concerned with detecting microbursts in real-time, they 
specify microbursts in terms of observable events, such as exceeding a configured buffer threshold in a specified time period~\cite{microburst-Cisco}. While microburst detection is orthogonal to the objectives of this paper, 
our proposed burstiness metrics can be adapted to detect bursts in 
real-time, by focusing only on short time intervals.

\section{Conclusions}
\label{sec:conclusion}

We presented measurement experiments that capture the training of a distributed DNN model in both a 
server-based and serverless setting and evaluated the burstiness properties of the traffic. 
The analysis of traffic burstiness was supported by burstiness metrics, that are suitable for 
determining whether a given traffic trace contains large rate spikes. 
The traffic of DNN training exhibited particular patterns that are unlike those seen in other application 
traffic.  In particular, DNN training generates traffic that consists of a series of burst transmissions that alternates with extended time periods  with few or no transmissions. 
The observed bursts  were  due to transmissions of gradients on the backward pass of DNN training, and the (almost) silent periods relate to the forward pass computations without transmissions. 
The gradient transmissions exhibited an extreme degree of short-term burstiness. At the same time, 
distributed ML systems were seen to mitigate the risk of congestion caused by a single  application  
by avoiding concurrent transmissions to the same destination, by either coordinating transmissions by  
workers to a common server or by having each worker transmit gradients to different destinations. 
This coordination, however, does not exist between concurrently running training applications, which 
 may cause congestion in DCNs. While this paper has raised awareness of the burstiness of distributed ML traffic, 
much more research is needed to investigate how DCN network topologies and traffic algorithms, 
accelerators such as GPUs and TPUs, different types of gradient aggregation, and different ML applications 
impact traffic burstiness. 

\vspace{1cm}
%\clearpage
%\bibliographystyle{IEEEtran}
%\bibliography{IEEEabrv,burst-refs}

% Generated by IEEEtran.bst, version: 1.14 (2015/08/26)

\newpage
\begin{appendices}

\section{Review of Ring Allreduce}
\label{sec:primer-ringallreduce}

\begin{figure}[h!]
\centering 
\subfloat[]{\includegraphics[width=0.22\textwidth]{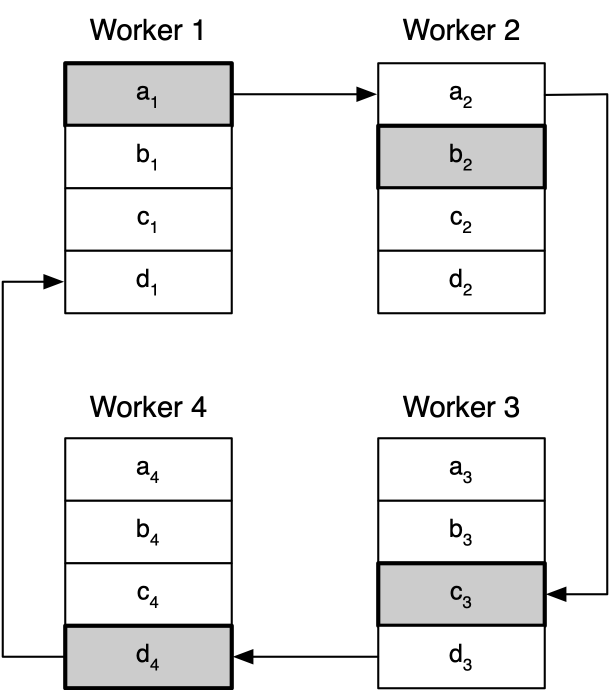}
\label{fig:Ringallreduce-1}}
\subfloat[]{\includegraphics[width=0.22\textwidth]{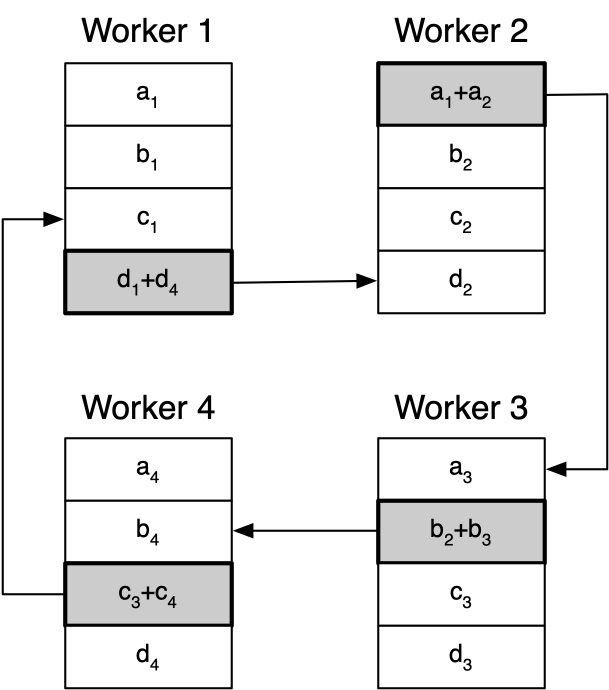}
\label{fig:Ringallreduce-2}}
%
%
%\centering
%
\subfloat[]{\includegraphics[width=0.22\textwidth]{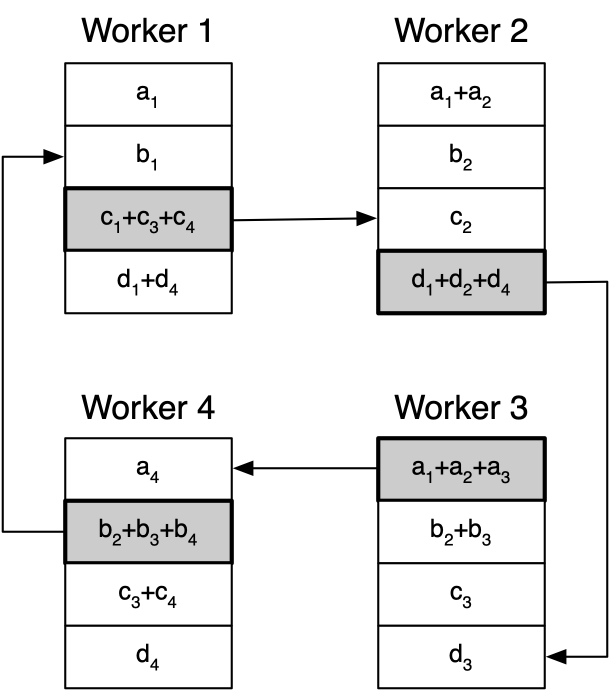}
\label{fig:Ringallreduce-3}}
\subfloat[]{\includegraphics[width=0.22\textwidth]{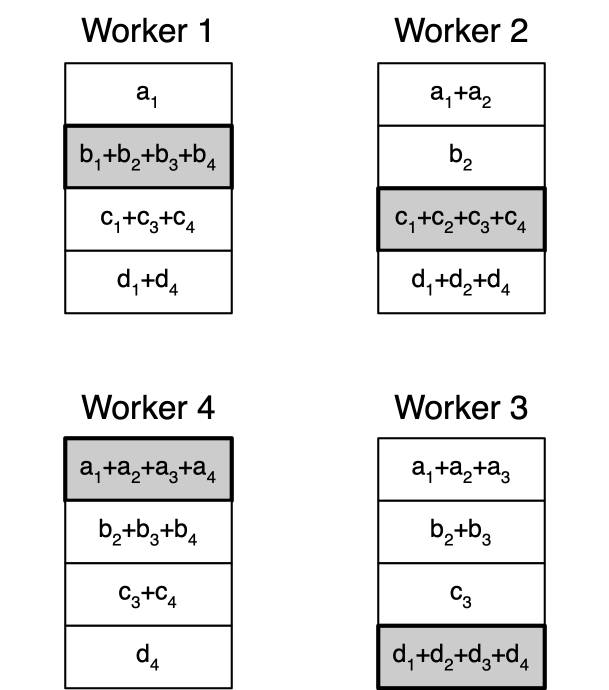}
\label{fig:Ringallreduce-4}}

\caption{Ring Allreduce: Reduce-Scatter phase.} 
\label{fig:ReduceScatter}
%\end{figure}
%%%------------------------------------------------------------------
%
%%-----------------------------------------------------------------
%\begin{figure}[t]
\centering 
\subfloat[]{\includegraphics[width=0.22\textwidth]{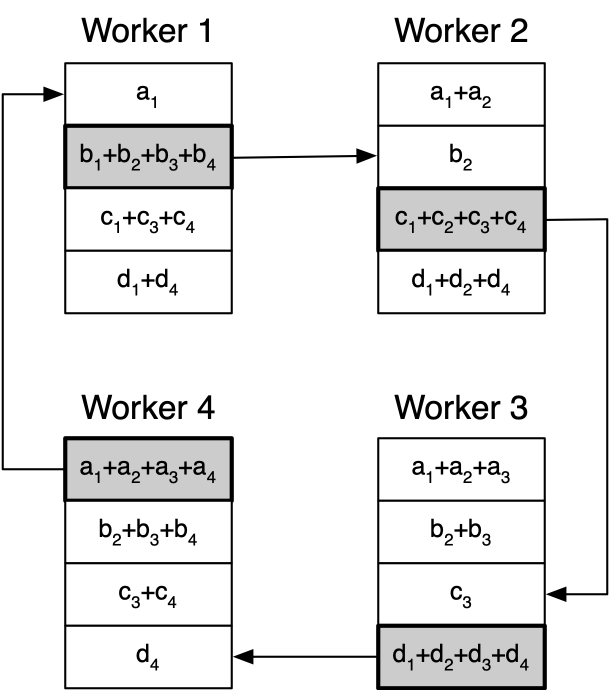}
\label{fig:Ringallreduce-5}}
\subfloat[]{\includegraphics[width=0.22\textwidth]{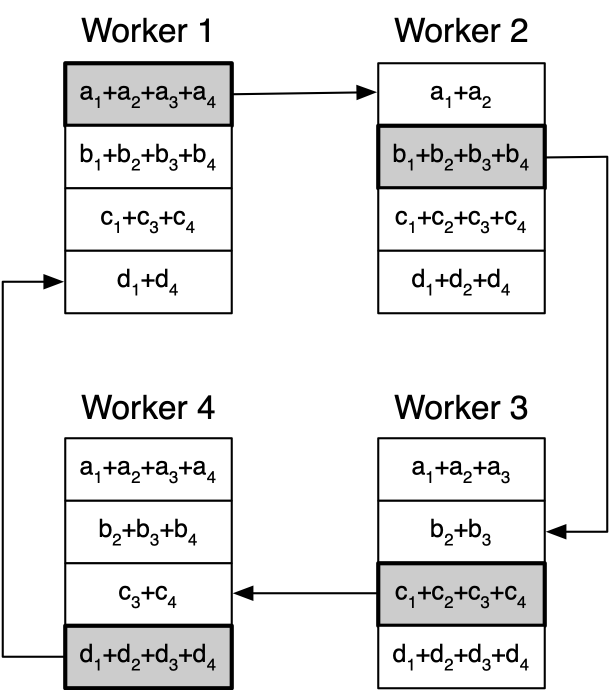}
\label{fig:Ringallreduce-6}}
%
%
%\centering
%
\subfloat[]{\includegraphics[width=0.22\textwidth]{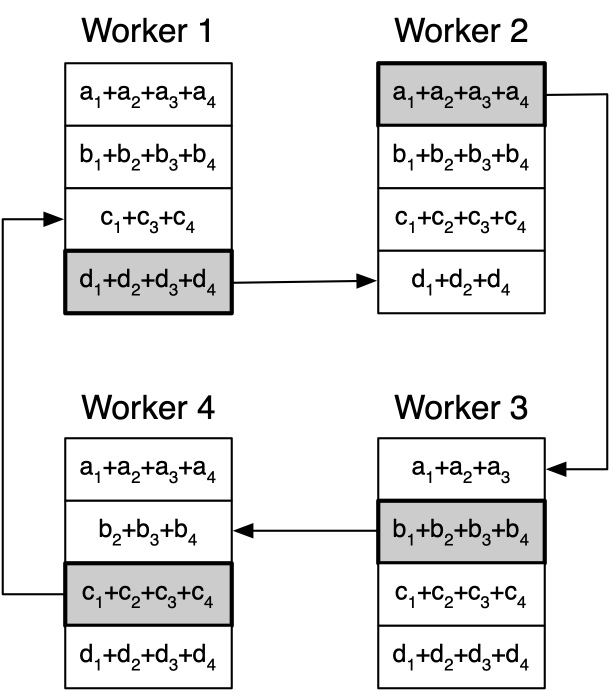}
\label{fig:Ringallreduce-7}}
\subfloat[]{\includegraphics[width=0.22\textwidth]{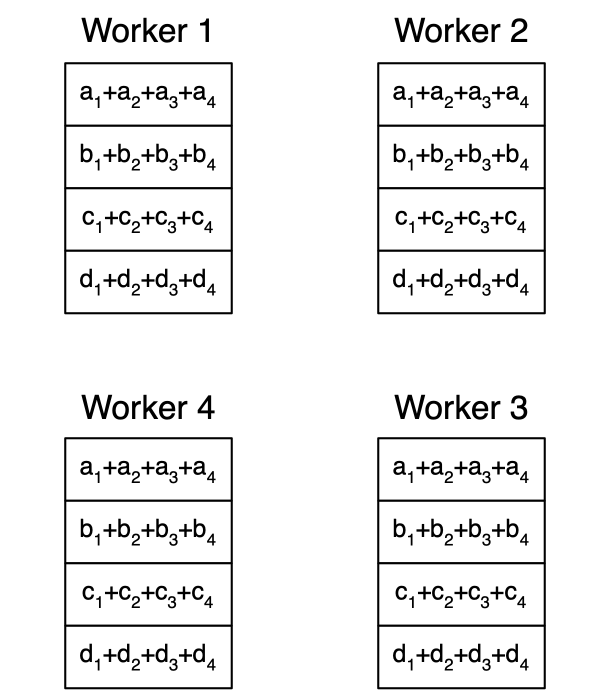}
\label{fig:Ringallreduce-8}}

\caption{Ring Allreduce: Allgather phase.} 
\label{fig:Allgather}
\end{figure}
%%------------------------------------------------------------------
%

In Ring Allreduce,  the parameter vector is partitioned and equally divided between  workers. An Allreduce operation is performed in two rounds of exchanges, referred to as Reduce-Scatter and Allgather. 
Figs.~\ref{fig:ReduceScatter} and~\ref{fig:Allgather} illustrate an Allreduce between four workers. Each worker~$j$ has a 
vector that consists of elements $a_j, b_j, c_j, d_j$. Suppose the desired  outcome of the depicted Allreduce is that each worker holds the sum 
of the elements from each worker. The Reduce-Scatter phase starts by having each worker 
send a different element to its successor in the logical ring (Fig.~\ref{fig:Ringallreduce-1}). The receiver of an element adds it to its own  and then sends the sum to it successor (Fig.~\ref{fig:Ringallreduce-2}). This  repeats until each worker holds the sum over all workers for one element  (Fig.~\ref{fig:Ringallreduce-4}). In the Allgather phase, the completed sums are  disseminated to all nodes. As before, this is done by having workers send data to their successor in the ring. 
First, each worker sends the element for which it holds  the complete sum (Fig.~\ref{fig:Ringallreduce-5}) to its successor. Upon receiving the data, the successor replaces the value of the element with the received sum. The Allgather  phase completes when each worker holds the complete sum of each element (Fig.~\ref{fig:Ringallreduce-8}). 

%------------------------------------------------------------------

\newpage
\section{Review of PFC and DCQCN}
\label{sec:dcqcn-pfc}

Increasingly, DCNs replace TCP by RDMA over Converged Ethernet~(RoCE) as network protocol.  Remote Direct Memory Access (RDMA) 
provides a memory-to-memory data exchange between remote hosts  
without involving host processors. 
RDMA is defined as part of the Infiniband standard~\cite{infiniband}.  
RoCEv2, which performs RDMA with UDP encapsulation~\cite{rocev2}, supports a `lossless Ethernet' service that avoids packet drops due to buffer overflows at switches. This is realized with the Priority-based Flow Control (PFC) protocol \cite{pfc}. Most RoCEv2 networks also provide 
end-to-end congestion control via the Data Center Quantized Congestion Notification (DCQCN) algorithm~\cite{DCQCN-Sigcomm2015}. 

PFC is a link-level flow control mechanism  between two switches or between a switch and a host. 
PFC separates a pair of link-layer endpoints into a sender and a receiver.  Either endpoint can be  a sender or a receiver depending on the direction of  data transmission. 
PFC  uses  
a backlog threshold $X_{\rm off}$ at an ingress port of a receiver to indicate that a link is congested. 
When this occurs, a receiver temporarily stops 
transmissions from the sender by sending it a  \emph{Pause} packet. By pausing transmissions sufficiently early, buffer overflows at the receiver can be fully prevented. 
Transmissions by the sender are paused until the backlog at the receiver reaches $X_{\rm on}$ ($X_{\rm on}< X_{\rm off})$.  
The necessary pause time until the backlog shrinks to $X_{\rm on}$ 
is included in the {\it Pause} packet.

DCQCN is an end-to-end  congestion control mechanism for RDMA that runs concurrently with PFC. It is a de-facto standard congestion control 
method for RoCEv2 traffic. 
The congestion control is based on Random Early Detection (RED) \cite{RED} for  detecting congestion and 
Explicit Congestion Notification (ECN) \cite{ECN} for reporting 
congestion.  RDMA uses so-called congestion notification packets (CNPs) 
to notify senders of congestion events. 
Marking of packets at a switch is governed by the RED algorithm, which marks packets with a  probability that is 
based on the backlog at an egress port. 
The arrival of a marked packet at a destination host indicates congestion in the network. If this happens the destination host issues 
a CNP to the source host, signaling that the source should 
reduce its rate. 
A CNP generation interval ($I_{\rm CNP}$) dictates the minimum elapsed time between two consecutive CNP transmissions (default value: 50\,$\mu$s).  
DCQCN reacts to congestion by reducing the transmission rate of the source host. In the absence of congestion, the transmission rate is increased. DCQCN overall follows an additive-increase-multiplicative-decrease approach for setting transmission rates. The details of the rate algorithm are 
relatively complex with about 15 configurable parameters. 
We refer to~\cite{DCQCN-Sigcomm2015} for a detailed discussion of 
the parameters. 

The configurations of PFC and DCQCN in the experimental setup (Sec.~\ref{subsec:testbed}) and  the simulation (Sec.~\ref{subsec:simulate-model-MB}) both use the parameters given in Table~\ref{table:PFC-DCQCN-parameters}. 

	\begin{table}[h!]
\caption{Parameter setting for PFC and DCQCN.}
\begin{center}
\begin{tabular}{| l | l |}
\hline 
$X_{\rm off}$ & 	9.5 kB/port/Gbps \\
$X_{\rm on}$ & 	9.25 kB/port/Gbps \\\hline 
$K_{\rm min}$ & 	7 kB \\
$K_{\rm max}$ & 	488 kB \\
$P_{\rm max}$ & 	30\% \\
$g$  & 	1/256 \\
$I_{\it CNP}$ & 50 $\mu$s \\
$K, T$ & 55 $\mu$s \\
$B$ & 10 MB \\
$R_{AI}$ & 5 Mbps \\
$R_{HI}$ & 50 Mbps\\
\hline 
\end{tabular}
\end{center}
\label{table:PFC-DCQCN-parameters}
\end{table}%
\end{appendices}

%\appendix{}
\end{document}